\documentclass{article}

\usepackage[nonatbib,preprint]{neurips_2019}

\usepackage[utf8]{inputenc} % allow utf-8 input
\usepackage[T1]{fontenc}    % use 8-bit T1 fonts
\usepackage{hyperref}       % hyperlinks
\usepackage{url}            % simple URL typesetting
\usepackage{booktabs}       % professional-quality tables
\usepackage{amsfonts}       % blackboard math symbols
\usepackage{nicefrac}       % compact symbols for 1/2, etc.
\usepackage{microtype}      % microtypography

\usepackage{amsmath}
\usepackage{amsthm}
\usepackage{algorithm}
\usepackage{algpseudocode}
\usepackage{bm}
\usepackage{graphicx}
\usepackage{placeins}
\usepackage{subfigure}
\usepackage{makecell}
\usepackage{multirow}
\usepackage{tabularx}
\usepackage{xcolor}
\usepackage[capitalize]{cleveref}
\usepackage{layouts}

\newtheorem{proposition}{Proposition}
\newtheorem*{proposition*}{Proposition}
\newtheorem{lemma}{Lemma}
\newtheorem*{lemma*}{Lemma}

\newtheorem*{theorem*}{Theorem}

\DeclareGraphicsExtensions{.png,.pdf}

\usepackage[citestyle=numeric,bibstyle=numeric, natbib=true,backend=bibtex,maxcitenames=2, maxbibnames=9]{biblatex}
\DeclareCiteCommand{\citenum}
  {}
  {\printfield[bibhyperref]{labelnumber}}
  {}
  {}
\newrobustcmd*{\parentexttrack}[1]{%
  \begingroup
  \blx@blxinit
  \blx@setsfcodes
  \blx@bibopenparen#1\blx@bibcloseparen
  \endgroup}
\AtEveryCite{%
  }

\newcommand{\eg}{\textit{e.g.}}
\newcommand{\ie}{\textit{i.e.}}

\newcommand{\wrt}{\textit{w.r.t.}}
\newcommand{\st}{\textit{s.t.}}

\newcommand{\ELBOTermM}{elbo}
\newcommand{\ELBOTerm}{\textit{\ELBOTermM}}
\newcommand{\ReconsTerm}{\textit{reconstruction loss}}
\newcommand{\KLTermM}{\E_{p^{\star}(\vv{x})}\KLD[q_{\phi}(\vv{z}|\vv{x})\|p_{\lambda}(\vv{z})]}
\newcommand{\KLTerm}{$\KLTermM$}
\newcommand{\KLzTermM}{ \KLD[q_{\phi}(\vv{z})\|p_{\lambda}(\vv{z})]}
\newcommand{\KLzTerm}{$\KLzTermM$}
\newcommand{\KLzxTermM}{\E_{p^{\star}(\vv{x})}\KLD[q_{\phi}(\vv{z}|\vv{x})\|p_{\theta}(\vv{z}|\vv{x})]}
\newcommand{\KLzxTerm}{$\KLzxTermM$}
\newcommand{\MIzxTermM}{\MI_{\phi}[Z;X]}
\newcommand{\MIzx}{$\MIzxTermM$}

\newcommand{\MNIST}{MNIST}
\newcommand{\StaticMNIST}{StaticMNIST}
\newcommand{\FashionMNIST}{FashionMNIST}
\newcommand{\Omniglot}{Omniglot}

\newcommand{\RealNVP}{RNVP}

\newcommand{\AggregatedPosterior}{aggregated posterior}
\newcommand{\AggregatedPosteriorI}{\textit{\AggregatedPosterior}}

\newcommand{\DenseVAE}{DenseVAE}
\newcommand{\ConvVAE}{ConvVAE}
\newcommand{\ResnetVAE}{ResnetVAE}
\newcommand{\PixelVAE}{PixelVAE}
\newcommand{\ConvOrResnetVAE}{Conv/ResnetVAE}

\newcommand{\TwoZMark}{\dagger}
\newcommand{\ThreeOrMoreZMark}{\ddagger}

\newcommand{\TheLastColumnSentence}[1]{The last column of each #1 grid show the images from the training set, most similar to the second-to-last column in pixel-wise L2 distance.}
  
\newcommand{\tablehr}[1]{#1}
\newcommand{\eqnref}[1]{\cref{#1}}
\newcommand{\UnderBraceNumber}[1]{{\small{\textcircled{\scriptsize{#1}}}}}
\newcommand{\myparagraph}[1]{{\textbf{#1}\hspace{.7em}}}

\newcommand\numberthis{\addtocounter{equation}{1}\tag{\theequation}}
\newcommand{\vv}[1]{\bm{\mathrm{#1}}}
\newcommand{\E}{\mathbb{E}}
\newcommand{\Var}{\mathrm{Var}}
\newcommand{\KLD}{\mathrm{D}_{KL}}
\newcommand{\MI}{\mathbb{I}}

\newcommand{\Dim}{\mathrm{Dim}}
\newcommand{\Det}[1]{\operatorname{\mathrm{det}}\left({#1}\right)}
\newcommand{\AbsDet}[1]{\left|\Det{#1}\right|}
\newcommand{\LogMeanExp}{\operatorname{\mathrm{LogMeanExp}}}
\newcommand{\argmin}{\operatorname{\arg\min}}
\newcommand{\dd}[1]{\mathrm{d}{#1}}
\newcommand{\IdentityMapping}{\mathrm{id}}
\newcommand{\myleft}{\mathopen{}\mathclose\bgroup\left}
\newcommand{\myright}{\aftergroup\egroup\right}
\newcommand{\Exp}[1]{\exp\myleft({#1}\myright)}

\newcommand{\NetLinear}{\mathrm{Linear}}
\newcommand{\NetDense}{\mathrm{Dense}}

\newcommand{\NetConv}{\mathrm{Conv}}
\newcommand{\NetLinearConv}{\mathrm{LinearConv}}
\newcommand{\NetResnet}{\mathrm{Resnet}}
\newcommand{\NetDeConv}{\mathrm{DeConv}}
\newcommand{\NetDeResnet}{\mathrm{DeResnet}}
\newcommand{\NetPixelCNN}{\mathrm{PixelCNN}}

\newcommand{\NetFlatten}{\mathrm{Flatten}}
\newcommand{\NetUnFlatten}{\mathrm{UnFlatten}}
\newcommand{\NetConcat}{\mathrm{Concat}}
\newcommand{\NetInvertibleDense}{\mathrm{InvertibleDense}}
\newcommand{\NetActNorm}{\mathrm{ActNorm}}
\newcommand{\NetCouplingLayer}{\mathrm{CouplingLayer}}
\newcommand{\NetTo}{\rightarrow}

\title{
	On the Necessity and Effectiveness of \\
	Learning the Prior of Variational Auto-Encoder
}

\author{%
  \begin{minipage}{\linewidth}
    \centering	
    \begin{minipage}{0.33\linewidth}
   	  \centering
  	  Haowen~Xu \\
  	  Zhihan~Li
    \end{minipage}
    \begin{minipage}{0.15\linewidth}
  	  \centering
  	  Wenxiao~Chen \\
  	  Youjian~Zhao
    \end{minipage}
    \begin{minipage}{0.33\linewidth}
  	  \centering
  	  Jinlin~Lai \\
  	  Dan~Pei
    \end{minipage}
  \end{minipage}\vspace{.25em}
  \\
  Beijing National Research Center for Information Science and Technology\\
  Tsinghua University\\
  Beijing, China\\
  \texttt{\{xhw15,chen-wx17,laijl16,lizhihan17\}@mails.tsinghua.edu.cn}\\
  \texttt{\{zhaoyoujian,peidan\}@tsinghua.edu.cn}
}

\begin{document}

\maketitle

\begin{abstract}
	Using powerful posterior distributions is a popular approach to achieving better variational inference.
	However, recent works showed that the aggregated posterior may fail to match unit Gaussian prior, thus learning the prior becomes an alternative way to improve the lower-bound.
	In this paper, for the first time in the literature, we prove the necessity and effectiveness of learning the prior when aggregated posterior does not match unit Gaussian prior, analyze why this situation may happen, and propose a hypothesis that learning the prior may improve reconstruction loss, all of which are supported by our extensive experiment results.
	We show that using learned Real NVP prior and just one latent variable in VAE, we can achieve test NLL comparable to very deep state-of-the-art hierarchical VAE, outperforming many previous works with complex hierarchical VAE architectures.
		%Further analysis of the experiments supports our theory.
%	Simple VAE architecture Learning the prior 
%	
%	Our discovery sheds light on enjoying lower test NLL, with more simplified VAE architectures.
%	Recent works showed that learning the prior can also improve the lower-bound.
%	However, recent works revealed the failure of the aggregated posterior matching the given Gaussian prior, while learning the prior can improve the lower-bound.
%	 we prove that learning the prior is a necessity when the aggregated posterior fails to match a fixed Gaussian prior, and we discuss why 
%	We find that, the key contribution of a learned prior is that, it can induce improved reconstruction loss.
%	We propose a novel explanation on why this happens.
%	Such effect of learned prior sheds a light to obtaining both lower test NLL and better reconstruction samples, with more simplified VAE architectures.
\end{abstract}

\section{Introduction}
%
%Variational auto-encoder (VAE) is a powerful deep generative model, capable of modeling complicated data~\citep{gregorConceptualCompression2016,gulrajaniPixelvaeLatentVariable2016,eslamiNeuralSceneRepresentation2018,akuzawaExpressiveSpeechSynthesis2018,gomez-bombarelliAutomaticChemicalDesign2018,maaloeBIVAVeryDeep2019}.
%The use of \textit{amortized variational inference} as well as the \textit{reparameterization trick}~\citep{kingmaAutoEncodingVariationalBayes2014,rezendeStochasticBackpropagationApproximate2014} makes VAE scalable to deep neural networks and large scale of data, which is one major reason of the success.

Variational auto-encoder (VAE)~\citep{kingmaAutoEncodingVariationalBayes2014,rezendeStochasticBackpropagationApproximate2014} is a powerful deep generative model.
The use of \textit{amortized variational inference} makes VAE scalable to deep neural networks and large amount of data.
Variational inference
%The variational inference method 
demands the intractable true posterior to be approximated by a tractable distribution.
The original VAE used factorized Gaussian for both the prior and the variational posterior~\citep{kingmaAutoEncodingVariationalBayes2014,rezendeStochasticBackpropagationApproximate2014}.
Since then, lots of more expressive variational posteriors have been proposed~\citep{tranVariationalGaussianProcess2015,rezendeVariationalInferenceNormalizing2015,salimansMarkovChainMonte2015,nalisnickApproximateInferenceDeep2016,kingmaImprovedVariationalInference2016,meschederAdversarialVariationalBayes2017,bergSylvesterNormalizingFlows2018}.
However, recent work suggested that even with powerful posteriors, VAE may still fail to match \AggregatedPosteriorI{} to unit Gaussian prior~\citep{roscaDistributionMatchingVariational2018}, indicating there is still a gap between the approximated and the true posterior.

To improve the lower-bound, one alternative to using powerful posterior distributions is  to learn the prior as well, an idea initially suggested by \citet{hoffmanElboSurgeryAnother2016}. 
%There have been a few works published along this direction.
Later on, \citet{huangLearnableExplicitDensity2017} applied Real NVP~\citep{dinhDensityEstimationUsing2016} to learn the prior.
\citet{tomczakVAEVampPrior2018} proved the optimal prior is the \AggregatedPosteriorI{}, which they approximate by assembling a mixture of the posteriors with a set of learned pseudo-inputs.
\citet{bauerResampledPriorsVariational2018} constructed a rich prior by multiplying a simple prior with a learned acceptance function.
\citet{takahashiVariationalAutoencoderImplicit2018} introduced the \textit{kernel density trick} to estimate the KL divergence in ELBO and log-likelihood, without explicitly learning the \AggregatedPosteriorI{}.

Despite the above works, no formal proof has been made to show the necessity and effectiveness of learning the prior.
Also, the previous works on prior fail to present comparable results with state-of-the-art VAE models, unless equipped with hierarchical latent variables, making it unclear whether the reported performance gain actually came from the learned prior, or from the complex architecture.
In this paper, we will discuss the necessity and effectiveness of learning the prior, and conduct comprehensive experiments on several datasets with learned Real NVP priors and just one latent variable.
Our contributions are:
\begin{itemize}
	\item We are the first to prove the necessity and effectiveness of learning the prior when \AggregatedPosteriorI{} does not match unit Gaussian prior, give novel analysis on why this situation may happen, and propose novel hypothesis that learning the prior can improve reconstruction loss, all of which are supported by our extensive experiment results.
	\item We conduct comprehensive experiments on four binarized datasets with four different network architectures. Our results show that VAE with Real NVP prior consistently outperforms standard VAE and Real NVP posterior.
	\item We are the first to show that using learned Real NVP prior with just one latent variable in VAE, it is possible to achieve test negative log-likelihoods (NLLs) comparable to very deep state-of-the-art hierarchical VAE  on these four datasets, outperforming many previous works using complex hierarchical VAE equipped with rich priors/posteriors.
	\item We demonstrate that the learned prior can avoid assigning high likelihoods to low-quality interpolations on the latent space and to the recently discovered low posterior samples~\citep{roscaDistributionMatchingVariational2018}.
\end{itemize}

\section{Preliminaries}
\label{sec:preliminaries}
\subsection{Variational auto-encoder}

Variational auto-encoder (VAE)~\citep{kingmaAutoEncodingVariationalBayes2014,rezendeStochasticBackpropagationApproximate2014} is a deep probabilistic model.
It uses a latent variable $\vv{z}$ with prior $p_{\lambda}(\vv{z})$, and a conditional distribution $p_{\theta}(\vv{x}|\vv{z})$, to model the observed variable $\vv{x}$.
The likelihood of a given $\vv{x}$ is formulated as $p_{\theta}(\vv{x}) = \int_{\mathcal{Z}} p_{\theta}(\vv{x}|\vv{z})\,p_{\lambda}(\vv{z}) \,\dd{\vv{z}}$,
%\begin{equation}
%    p_{\theta}(\vv{x}) = \int_{\mathcal{Z}} p_{\theta}(\vv{x}|\vv{z})\,p_{\lambda}(\vv{z}) \,\mathrm{d}\vv{z}	\label{eqn:vae-formulation}
%\end{equation}
where $p_{\theta}(\vv{x}|\vv{z})$ is typically derived by a neural network with learnable parameter $\theta$.  Using variational inference, the log-likelihood $\log p_{\theta}(\vv{x})$ is bounded below by evidence lower-bound (ELBO) of $\vv{x}$~\eqref{eqn:elbo}:
\begin{align*}
    \log p_{\theta}(\vv{x}) &\geq \log p_{\theta}(\vv{x}) - \KLD\left[q_{\phi}(\vv{z}|\vv{x})\|p_{\theta}(\vv{z}|\vv{x})\right] \\
    &= \E_{q_{\phi}(\vv{z}|\vv{x})}\left[ \log p_{\theta}(\vv{x}|\vv{z}) + \log p_{\lambda}(\vv{z}) - \log q_{\phi}(\vv{z}|\vv{x}) \right] \\
    &= \mathcal{L}(\vv{x};\lambda,\theta,\phi) \numberthis\label{eqn:elbo} \\
    &= \E_{q_{\phi}(\vv{z}|\vv{x})}\left[ \log p_{\theta}(\vv{x}|\vv{z})\right] - \KLD\left[q_{\phi}(\vv{z}|\vv{x})\|p_{\lambda}(\vv{z})\right] \numberthis\label{eqn:elbo-1}
\end{align*}
where $q_{\phi}(\vv{z}|\vv{x})$ is the variational posterior to approximate $p_{\theta}(\vv{z}|\vv{x})$, derived by a neural network with parameter $\phi$.
\eqnref{eqn:elbo-1} is one decomposition of \eqref{eqn:elbo}, where the first term is the \textit{reconstruction loss} of $\vv{x}$, and the second term is the Kullback Leibler (KL) divergence between $q_{\phi}(\vv{z}|\vv{x})$ and $p_{\lambda}(\vv{z})$.

Optimizing $q_{\phi}(\vv{z}|\vv{x})$ and $p_{\theta}(\vv{x}|\vv{z})$ \wrt{} empirical distribution $p^{\star}(\vv{x})$ can be achieved by maximizing ``the expected ELBO \wrt{} the empirical distribution $p^{\star}(\vv{x})$'' (denoted by \ELBOTerm{} for short hereafter):
\begin{equation}
	\mathcal{L}(\lambda,\theta,\phi) = \E_{p^{\star}(\vv{x})}[\mathcal{L}(\vv{x};\lambda,\theta,\phi)]= \E_{p^{\star}(\vv{x})}\,\E_{q_{\phi}(\vv{z}|\vv{x})} \left[\log p_{\theta}(\vv{x}|\vv{z}) + \log p_{\lambda}(\vv{z}) - \log q_{\phi}(\vv{z}|\vv{x}) \right] \label{eqn:elbo-sum}
\end{equation}
\subsection{Real NVP}

Real NVP~\citep{dinhDensityEstimationUsing2016} (\textit{\RealNVP{}} for short hereafter) is also a deep probabilistic model, and
we denote its observed variable by $\vv{z}$ and latent variable by $\vv{w}$, with marginal distribution $p_{\lambda}(\vv{z})$ and prior $p_{\xi}(\vv{w})$. 
%To be consistent with \RealNVP{} prior in later text, 
RNVP relates $\vv{z}$ and $\vv{w}$ by an invertible mapping $\vv{w} = f_{\lambda}(\vv{z})$, instead of a conditional distribution in VAE.
Given the invertibility of $f_{\lambda}$, we have:
%\begin{align}
%	p_{\lambda}(\vv{z}) &= p_{\xi}(\vv{w}) \,\AbsDet{\frac{\partial f_{\lambda}(\vv{z})}{\partial\vv{z}}} \label{eqn:real-nvp-px} \\
%	\vv{z} &= f^{-1}_{\lambda}(\vv{w}) \label{eqn:real-nvp-finv}
%\end{align}
\begin{equation}
	p_{\lambda}(\vv{z}) = p_{\xi}(\vv{w}) \,\AbsDet{\frac{\partial f_{\lambda}(\vv{z})}{\partial\vv{z}}}, \quad \vv{z} = f^{-1}_{\lambda}(\vv{w}) \label{eqn:real-nvp-px}
\end{equation}
where $\Det{\partial f_{\lambda}(\vv{z}) / \partial\vv{z}}$ is the Jacobian determinant of $f_{\lambda}$.
%\eqnref{eqn:real-nvp-finv} can be used to sample $\vv{z}$ by ancestral sampling, starting from the prior $p_{\xi}(\vv{w})$.
In \RealNVP{}, $f_{\lambda}$ is composed of $K$ invertible mappings, where $f_{\lambda}(\vv{z}) = (f_K \circ \dots \circ f_1)(\vv{z})$, and each $f_k$ is invertible.
%If we let $\vv{h}_0 = \vv{z}$, $\vv{h}_K = \vv{w}$, and the output of each intermediate mapping $\vv{h}_{k} = f_k(\vv{h}_{k-1})$, we can further get:
%\begin{align}
%	p_{\lambda}(\vv{z}) &= p_{\xi}(\vv{w}) \, \prod_{k=1}^K \AbsDet{\frac{\partial f_k(\vv{h}_{k-1})}{\partial \vv{h}_{k-1}}} \\
%	\vv{z} &= (f^{-1}_1 \circ \dots \circ f^{-1}_K)(\vv{w})
%\end{align}
%\begin{equation}
%	p_{\lambda}(\vv{z}) = p_{\xi}(\vv{w}) \, \prod_{k=1}^K \AbsDet{\frac{\partial f_k(\vv{h}_{k-1})}{\partial \vv{h}_{k-1}}}, \quad
%	\vv{z} = (f^{-1}_1 \circ \dots \circ f^{-1}_K)(\vv{w})
%\end{equation}
$f_k$ must be carefully designed to ensure that the determinant can be computed efficiently.
The original paper of \RealNVP{} introduced the \textit{affine coupling layer} as $f_k$.
\citet{kingmaGlowGenerativeFlow2018} further introduced \textit{actnorm} and \textit{invertible 1x1 convolution}.
Details can be found in their respective papers.

\section{Learning the prior with \RealNVP{}}
%
%\paragraph{Formulation}

It is straightforward to obtain a rich prior $p_{\lambda}(\vv{z})$ from a simple (\ie{}, with constant parameters) one with \RealNVP{}.
Denote the simple prior as $p_{\xi}(\vv{w})$, while the \RealNVP{} mapping as $\vv{w} = f_{\lambda}(\vv{z})$.
We then obtain \cref{eqn:real-nvp-px} as our prior $p_{\lambda}(\vv{z})$.
%\begin{align}
%	p_{\lambda}(\vv{z}) &= p_{\xi}(\vv{w}) \AbsDet{\frac{\partial f_{\lambda}(\vv{z})}{\partial\vv{z}}} \label{eqn:real-nvp-px} \\
%	\vv{z} &= f^{-1}_{\lambda}(\vv{w}) \label{eqn:prior-finv}
%\end{align}
%\begin{equation}
%	p_{\lambda}(\vv{z}) = p_{\xi}(\vv{w}) \AbsDet{\frac{\partial f_{\lambda}(\vv{z})}{\partial\vv{z}}} , \quad \vv{z} = f^{-1}_{\lambda}(\vv{w}) \label{eqn:real-nvp-px}
%\end{equation}
%Sampling $\vv{x}$ from a trained VAE with \RealNVP{} prior can be achieved by first sampling $\vv{w}$, then using \eqref{eqn:prior-finv} to obtain $\vv{z}$, and finally sampling $\vv{x}$ from $p_{\theta}(\vv{x}|\vv{z})$.
Substitute \eqnref{eqn:real-nvp-px} into \eqref{eqn:elbo-sum}, we get to the training objective: %can obtain the ELBO for VAE with Real NVP prior:
\begin{equation}
	\mathcal{L}(\lambda,\theta,\phi) = 
		\E_{p^{\star}(\vv{x})}\,\E_{q_{\phi}(\vv{z}|\vv{x})}\left[ \log p_{\theta}(\vv{x}|\vv{z}) + \log p_{\xi}(f_{\lambda}(\vv{z})) + \log  \AbsDet{\frac{\partial f_{\lambda}(\vv{z})}{\partial\vv{z}}} - \log q_{\phi}(\vv{z}|\vv{x}) \right] \label{eqn:elbo-rnvp-pz}
\end{equation}

We mainly use \textbf{joint training}~\citep{tomczakVAEVampPrior2018,bauerResampledPriorsVariational2018} (where $p_{\lambda}(\vv{z})$ is jointly trained along with $q_{\phi}(\vv{z}|\vv{x})$ and $p_{\theta}(\vv{x}|\vv{z})$ by directly maximizing \eqnref{eqn:elbo-rnvp-pz}) in our experiments, but we also consider the other two training strategies: \textbf{1) post-hoc training}~\citep{bauerResampledPriorsVariational2018}, where $p_{\lambda}(\vv{z})$ is optimized to match $q_{\phi}(\vv{z})$ of a standard, pre-trained VAE; and \textbf{2) iterative training} (proposed by us), where we alternate \textit{between} training $p_{\theta}(\vv{x}|\vv{z})$ \& $q_{\phi}(\vv{z}|\vv{x})$ \textit{and} training $p_{\lambda}(\vv{z})$, for multiple iterations.
More details can be found in \cref{sec:appendix-training-evaluation}.

\section{The necessity of learning the prior}
The \AggregatedPosteriorI{}, defined as $q_{\phi}(\vv{z}) = \int_{\mathcal{Z}} q_{\phi}(\vv{z}|\vv{x})\,p^{\star}(\vv{x})\,\mathrm{d}\vv{z}$, should be equal to $p_{\lambda}(\vv{z})$, if VAE is \textit{perfectly trained}, \ie{}, $q_{\phi}(\vv{z}|\vv{x}) \equiv p_{\theta}(\vv{z}|\vv{x})$ and $p_{\theta}(\vv{x}) \equiv p^{\star}(\vv{x})$~\citep{hoffmanElboSurgeryAnother2016}.
However, \citet{roscaDistributionMatchingVariational2018} showed that even with powerful posteriors, the \AggregatedPosteriorI{} may still  not match a unit Gaussian prior, which, we argue, is a practical limitation of neural networks and the existing optimization techniques. To better match the prior and \AggregatedPosteriorI{},
\citet{hoffmanElboSurgeryAnother2016} suggested a decomposition of \ELBOTerm{} (\eqnref{eqn:elbo-sum}):
\begin{equation}
	\mathcal{L}(\lambda,\theta,\phi) =
		\underbrace{\E_{p^{\star}(\vv{x})}\,\E_{q_{\phi}(\vv{z}|\vv{x})}\left[ \log p_{\theta}(\vv{x}|\vv{z})\right]}_{\UnderBraceNumber{1}} - \underbrace{\KLD\left[q_{\phi}(\vv{z})\|p_{\lambda}(\vv{z})\right]}_{\UnderBraceNumber{2}} - \underbrace{\MI_{\phi}[Z;X]}_{\UnderBraceNumber{3}} \label{eqn:elbo-2}
\end{equation}
where \UnderBraceNumber{3} is the \textit{mutual information}, defined as $\MI_{\phi}[Z;X] = \iint q_{\phi}(\vv{z},\vv{x}) \log \frac{q_{\phi}(\vv{z},\vv{x})}{q_{\phi}(\vv{z})\,p^{\star}(\vv{x})} \dd{\vv{z}}\,\dd{\vv{x}}$.
Since $p_{\lambda}(\vv{z})$ is only included in \UnderBraceNumber{2}, \ELBOTerm{} can be further enlarged if $p_{\lambda}(\vv{z})$ is trained to match $q_{\phi}(\vv{z})$.
However, neither the existence of a better $p_{\lambda}(\vv{z})$, nor the necessity of learning the $p_{\lambda}(\vv{z})$ for reaching the extremum of \ELBOTerm{}, has been proved under the condition that $q_{\phi}(\vv{z})$ does not match $\mathcal{N}(\vv{0},\vv{I})$.
In the following, we shall give the proof, and discuss why $q_{\phi}(\vv{z})$ could not match $\mathcal{N}(\vv{0},\vv{I})$ in practice.
%Given the condition $q_{\phi}(\vv{z}|\vv{x})$, $p_{\theta}(\vv{x}|\vv{z})$ and $f_{\lambda}(\vv{z})$ are continuous and differentiable, we have:
%We shall prove the following proposition:
\begin{proposition}\label{proposition-1}
	For VAE with flow prior $p_{\lambda}(\vv{z}) = p_{\mathcal{N}}(f_{\lambda}(\vv{z}))\AbsDet{\partial f_{\lambda}/\partial \vv{z}}$, where $p_{\mathcal{N}}(f_{\lambda}(\vv{z})) = \mathcal{N}(\vv{0},\vv{I})$, if $q_{\phi}(\vv{z}) \neq \mathcal{N}(\vv{0},\vv{I})$, then its training objective (\ELBOTerm{}):
\begin{equation*}
	\mathcal{L}(\lambda,\theta,\phi) = \E_{p^{\star}(\vv{x})}\,\E_{q_{\phi}(\vv{z}|\vv{x})}\left[ \log p_{\theta}(\vv{x}|\vv{z}) + \log p_{\mathcal{N}}(f_{\lambda}(\vv{z})) + \log  \AbsDet{\frac{\partial f_{\lambda}(\vv{z})}{\partial\vv{z}}} - \log q_{\phi}(\vv{z}|\vv{x}) \right]
\end{equation*}
can reach its extremum only if $f_{\lambda} \neq f_{\lambda_0}$, where $f_{\lambda_0} = \IdentityMapping$ is the identity mapping.
Also, $\forall \theta, \phi$, if $q_{\phi}(\vv{z}) \neq \mathcal{N}(\vv{0},\vv{I})$, then there exist $f_{\lambda} \neq f_{\lambda_0}$, \st{} $\mathcal{L}(\lambda,\theta,\phi) > \mathcal{L}(\lambda_0,\theta,\phi)$.

\end{proposition}
\begin{proof}
	See \cref{sec:appendix-proposition-1-proof}.	
\end{proof}
Since $\mathcal{L}(\lambda_0,\theta,\phi)$ is exactly the training objective for a standard VAE with unit Gaussian prior, we conclude that when $q_{\phi}(\vv{z}) \neq \mathcal{N}(\vv{0},\vv{I})$, learning the prior is necessary, and there always exists a $f_{\lambda}$ that gives us a higher \textit{elbo} than just using a unit Gaussian prior.
 
To analyze why $q_{\phi}(\vv{z})$ does not match the unit Gaussian prior, we start with the following proposition:
\begin{proposition}\label{proposition-2}
	Given a finite number of discrete training data, \ie{}, $p^{\star}(\vv{x}) = \frac{1}{N} \sum_{i=1}^N \delta(\vv{x}-\vv{x}^{(i)})$, if $p_{\theta}(\vv{x}|\vv{z}) = \mathrm{Bernoulli}(\vv{\mu}_{\theta}(\vv{z}))$, where the Bernoulli mean $\vv{\mu}_{\theta}(\vv{z})$ is produced by the decoder, and $0 < \mu^k_{\theta}(\vv{z}) < 1$ for each of its $k$-th dimensional output,
%	For a finite number of discrete training data, \ie{}, the training distribution $p^{\star}(\vv{x}) = \frac{1}{N} \sum_{i=1}^N \delta(\vv{x}-\vv{x}^{(i)})$, 
%	if we let $p_{\theta}(\vv{x}|\vv{z}) = \mathrm{Bernoulli}(\vv{\mu}_{\theta}(\vv{z}))$, where the Bernoulli mean $\vv{\mu}_{\theta}(\vv{z})$ is produced by the decoder, $0 < \mu^k_{\theta}(\vv{z}) < 1$ for each of its $k$-th dimensional output, %$x_k$ is the $k$-th dimension of $\vv{x}$ and $\mu^{k}_{\theta}(\vv{z})$ is the $k$-th dimension of $\vv{\mu}_{\theta}(\vv{z})$, $0 < \mu^{k}_{\theta}(\vv{z}) < 1$, 
	then 
	%for any encoder $q_{\phi}(\vv{z}|\vv{x})$ and prior $p_{\lambda}(\vv{z})$, 
	the optimal decoder $\vv{\mu}_{\theta}(\vv{z})$ is:
	\begin{equation}
		\vv{\mu}_{\theta}(\vv{z}) = \sum_i w_i(\vv{z})\,\vv{x}^{(i)}, \quad \text{where} \;\,
		w_i(\vv{z}) = \frac{q_{\phi}(\vv{z}|\vv{x}^{(i)})}{\sum_j q_{\phi}(\vv{z}|\vv{x}^{(j)})} \;\, \text{and}\;\, \sum_i w_i(\vv{z}) = 1 \label{eqn:appendix-bernoulli-px-given-z-extermum}
	\end{equation}
\end{proposition}
\begin{proof}
	See \cref{sec:appendix-bernoulli-px-given-z-extermum}.
\end{proof}

\cref{proposition-2} suggests that if $q_{\phi}(\vv{z}|\vv{x})$ for different $\vv{x}$ overlap, then even at the center of $q_{\phi}(\vv{z}|\vv{x}^{(i)})$ of one training point $\vv{x}^{(i)}$, the optimal decoder will be an \textit{average} of both $\vv{x}^{(i)}$ and other training points $\vv{x}^{(j)}$, weighted by $w_i(\vv{z})$ and $w_j(\vv{z})$.
%This is exactly the case for Gaussian posterior and other continuous posteriors like Real NVP prior.
However, \citet{rezendeTamingVAEs2018} has shown that weighted average like this is likely to cause poor \textit{reconstruction loss} (\UnderBraceNumber{1} in \eqnref{eqn:elbo-2}) and ``blurry reconstruction''.
One direction to optimize \cref{eqn:elbo-2} is to enlarge \UnderBraceNumber{1}, and to achieve this, it is a crucial goal to reduce the weight $w_j(\vv{z}), j \neq i$ for $\vv{z}$ near the center of every $q_{\phi}(\vv{z}|\vv{x}^{(i)})$.
%To enlarge \UnderBraceNumber{1}, it is crucial to reduce the weight $w_j(\vv{z}), j \neq i$ for $\vv{z}$ at the center of every $q_{\phi}(\vv{z}|\vv{x}^{(i)})$.

In order to achieve this goal for Gaussian posterior, one way is to reduce the standard deviations (std) of $q_{\phi}(\vv{z}|\vv{x})$ for all $\vv{x}$.
But smaller stds for all $\vv{x}$ is likely to cause \AggregatedPosteriorI{} $q_{\phi}(\vv{z})$ to become more dissimilar with unit Gaussian prior, \ie{}, enlarging \UnderBraceNumber{2} in \eqnref{eqn:elbo-2}.
%Also, better \textit{reconstruction loss} typically means larger 
% \textit{mutual information} between $\vv{x}$ and $\vv{z}$ 
%\textit{mutual information} (\UnderBraceNumber{3} in \eqnref{eqn:elbo-2}), see \cref{sec:appendix-reconstruction-loss-and-mutual-information}.
Also, there is trade-off between \textit{reconstruction loss} and \textit{mutual information}, see \cref{sec:appendix-reconstruction-loss-and-mutual-information}.
From above analysis, we can see that, at least for Gaussian posterior, \textbf{there is a trade-off between \UnderBraceNumber{1} and (\UnderBraceNumber{2} + \UnderBraceNumber{3}) in \eqnref{eqn:elbo-2}.
Due to this trade-off, \UnderBraceNumber{2} is extremely hard to be optimal (\ie{}, $q_{\phi}(\vv{z}) \equiv p_{\lambda}(\vv{z})$).
We think this is one important reason why $q_{\phi}(\vv{z})$ cannot match $p_{\lambda}(\vv{z})$ in practice.}

$p_{\lambda}(\vv{z})$ appears only in \UnderBraceNumber{2}.
With learned prior, the influence of \UnderBraceNumber{2} on the training objective (\cref{eqn:elbo-2}) is much smaller, thus the trade-off seems to occur mainly between \UnderBraceNumber{1} and \UnderBraceNumber{3}.
Since the numerical values of \UnderBraceNumber{1} is typically much larger than \UnderBraceNumber{3} in practice, the new trade-off is likely to cause \UnderBraceNumber{1} to increase.
Thus, we propose the hypothesis that a learned prior may improve \textit{reconstruction loss}.
%, in addition to lower $\KLD[q_{\phi}(\vv{z})\|p_{\lambda}(\vv{z})]$ (\UnderBraceNumber{2}).

Although learning the prior can reduce \UnderBraceNumber{2}, a learned prior does not necessarily make KL divergence smaller (as opposed to what was previously implied by~\citet{bauerResampledPriorsVariational2018}), since the KL divergence $\KLTermM = \UnderBraceNumber{2} + \UnderBraceNumber{3}$ is also affected by the \textit{mutual information} (\UnderBraceNumber{3}).
%Meanwhile, previous works have used \textit{active units}~\citep{burdaImportanceWeightedAutoencoders2015} to quantify the number of used latent dimensions for encoding information from input data.
%Since a larger \textit{reconstruction loss} typically means larger \textit{mutual information}, learned prior may also increase the number of \textit{active units}.

\citet{rezendeTamingVAEs2018} has proved when $p_{\theta}(\vv{x}|\vv{z}) = \mathcal{N}(\vv{\mu}_{\theta}(\vv{z}),\sigma^2 \vv{I})$, where $\sigma$ is a fixed constant, the optimal decoder is also $\vv{\mu}_{\theta}(\vv{z}) = \sum_i w_i(\vv{z})\,\vv{x}^{(i)}$, thus our analysis naturally holds in such situation.

For non-Gaussian posteriors, as long as $q_{\phi}(\vv{z}|\vv{x})$ is defined on the whole $\mathbb{R}^n$ (\eg{}, flow posteriors derived by applying continuous mappings on $\mathcal{N}(\vv{0},\vv{I})$), it is possible that $w_j(\vv{z})$ for $\vv{z}$ near the center of $q_{\phi}(\vv{z}|\vv{x}^{(i)})$ is not small enough, causing ``blurry reconstruction''.
Learning the prior may help optimize such posteriors to produce a better $w_j(\vv{z})$.
We also suspect this problem may occur with other element-wise $p_{\theta}(\vv{x}|\vv{z})$.
These concerns highlight the necessity of learning the prior.

\section{Experiments}
\begin{table}[t]
	\centering
	\caption{Average test NLL (lower is better) of different models, with Gaussian prior \& Gaussian posterior (``standard''), Gaussian prior \& \RealNVP{} posterior (``RNVP $q(z|x)$''), and \RealNVP{} prior \& Gaussian posterior (``RNVP $p(z)$'').  The flow depth $K$ is 20 for \RealNVP{} priors and posteriors.}\label{tbl:main-experiments-nll-mean-only}
	\small
	\bgroup\setlength{\tabcolsep}{.44em}
\begin{tabular}{cccccccccc}
  \toprule & \multicolumn{3}{c}{\tablehr{\DenseVAE{}}}  & \multicolumn{3}{c}{\tablehr{\ResnetVAE{}}}  & \multicolumn{3}{c}{\tablehr{\PixelVAE{}}}  \\
 \cmidrule(lr){2-4}  \cmidrule(lr){5-7}  \cmidrule(lr){8-10}   \tablehr{Datasets} & standard & \parbox[c]{3em}{\centering RNVP \\ \(q(z|x)\)} & \parbox[c]{3em}{\centering RNVP \\ \(p(z)\)} & standard & \parbox[c]{3em}{\centering RNVP \\ \(q(z|x)\)} & \parbox[c]{3em}{\centering RNVP \\ \(p(z)\)} & standard & \parbox[c]{3em}{\centering RNVP \\ \(q(z|x)\)} & \parbox[c]{3em}{\centering RNVP \\ \(p(z)\)} \\
  \midrule
    \StaticMNIST{}
 &
88.84 &
86.07 &
\textbf{84.87} &
82.95 &
80.97 &
\textbf{79.99} &
79.47 &
79.09 &
\textbf{78.92}    \\
    \MNIST{}
 &
84.48 &
82.53 &
\textbf{80.43} &
81.07 &
79.53 &
\textbf{78.58} &
78.64 &
78.41 &
\textbf{78.15}    \\
    \FashionMNIST{}
 &
228.60 &
227.79 &
\textbf{226.11} &
226.17 &
225.02 &
\textbf{224.09} &
224.22 &
223.81 &
\textbf{223.40}    \\
    \Omniglot{}
 &
106.42 &
102.97 &
\textbf{102.19} &
96.99 &
94.30 &
\textbf{93.61} &
89.83 &
89.69 &
\textbf{89.61}    \\
\bottomrule
\end{tabular}        \egroup
\end{table}

\begin{table}[t]
	\centering
	\caption{Average test NLL of \ResnetVAE{} + \RealNVP{} prior with different flow depth $K$}\label{tbl:flow-depth-nll-mean-only}
		\small
    \bgroup\setlength{\tabcolsep}{0.85em}
\begin{tabular}{ccccccccc}
  \toprule
  & \multicolumn{8}{c}{\tablehr{Flow depth $K$ for \RealNVP{} prior}} \\
  \cmidrule(lr){2-9}
  \tablehr{Datasets}  & 0  & 1  & 2  & 5  & 10  & 20  & 30  & 50  \\
  \midrule
    \StaticMNIST{}
      & 82.95      & 81.76      & 81.30      & 80.64      & 80.26      & 79.99      & 79.90      & \textbf{79.84}    \\
    \MNIST{}
      & 81.07      & 80.02      & 79.58      & 79.09      & 78.75      & 78.58      & 78.52      & \textbf{78.49}    \\
    \FashionMNIST{}
      & 226.17      & 225.27      & 224.78      & 224.37      & 224.18      & 224.09      & \textbf{224.07}      & \textbf{224.07}    \\
    \Omniglot{}
      & 96.99      & 96.20      & 95.35      & 94.47      & 93.92      & 93.61      & 93.53      & \textbf{93.52}    \\
  \bottomrule
\end{tabular}
        \egroup
\end{table}

\begin{table}[t]
	\centering
	\begin{minipage}[t]{.48\linewidth}
		\centering
		\caption{Average test NLL of \ResnetVAE{}, with \RealNVP{} posterior (``RNVP $q(z|x)$''), \RealNVP{} prior (``RNVP $p(z)$''), and both \RealNVP{} prior \& posterior (``both'').  Flow depth $K=20$.}\label{tbl:dual-rnvp-nll}
		\small
		\bgroup\setlength{\tabcolsep}{.75em}
\begin{tabular}{cccc}
  \toprule
  & \multicolumn{3}{c}{\tablehr{\ResnetVAE}} \\
  \cmidrule(lr){2-4}
  \tablehr{Datasets}  & \parbox[c]{3em}{RNVP \\ \centering \(q(z|x)\)}  & \parbox[c]{3em}{RNVP \\ \centering \(p(z)\)}  & both  \\
  \midrule
    \StaticMNIST{} 
 &
80.97 &
79.99 &
\textbf{79.87}    \\
    \MNIST{} 
 &
79.53 &
78.58 &
\textbf{78.56}    \\
    \FashionMNIST{} 
 &
225.02 &
224.09 &
\textbf{224.08}    \\
    \Omniglot{} 
 &
94.30 &
\textbf{93.61} &
93.68    \\
  \bottomrule
\end{tabular}
        \egroup
	\end{minipage}\hfill
	\begin{minipage}[t]{.48\linewidth}
		\centering
		\caption{Average test NLL of \ResnetVAE{}, with prior trained by: \textit{joint} training, \textit{iterative} training, \textit{post-hoc} training, and standard VAE (``none'') as reference.  Flow depth $K=20$.}\label{tbl:post-train-nll}\vspace{.5em}
		\small{
    	\bgroup\setlength{\tabcolsep}{.25em}
\begin{tabular}{ccccc}
  \toprule
  & \multicolumn{4}{c}{\tablehr{\ResnetVAE}} \\
  \cmidrule(lr){2-5}
  \tablehr{Datasets}  & joint  & iterative  & post-hoc  & none  \\
  \midrule
    \StaticMNIST{} 
 &
\textbf{79.99} &
80.63 &
80.86 &
82.95    \\
    \MNIST{} 
 &
\textbf{78.58} &
79.61 &
79.90 &
81.07    \\
    \FashionMNIST{} 
 &
\textbf{224.09} &
224.88 &
225.22 &
226.17    \\
    \Omniglot{} 
 &
\textbf{93.61} &
94.43 &
94.87 &
96.99    \\
  \bottomrule
\end{tabular}
        \egroup
    	}
	\end{minipage}
\end{table}

\begin{table}[t]
	\centering
	\begin{minipage}[t]{.48\linewidth}
		\flushleft
		\caption{Test NLL on \StaticMNIST{}. ``$^\TwoZMark$'' indicates a hierarchical model with 2 latent variables, while ``$^\ThreeOrMoreZMark$'' indicates at least 3 latent variables.}\label{tbl:static-mnist-nll-comparison}
		\small
		
\begin{tabularx}{\linewidth}{Xr}
  \toprule
  Model & NLL \\
    \midrule
\multicolumn{2}{l}{\textit{Models without PixelCNN decoder}} \\    ConvHVAE + Lars prior$^\TwoZMark$~[\citenum{bauerResampledPriorsVariational2018}] & 81.70 \\
    ConvHVAE + VampPrior$^\TwoZMark$~[\citenum{tomczakVAEVampPrior2018}] & 81.09 \\
    VAE + IAF$^\ThreeOrMoreZMark$~[\citenum{kingmaImprovedVariationalInference2016}] & 79.88 \\
    BIVA$^\ThreeOrMoreZMark$~[\citenum{maaloeBIVAVeryDeep2019}] & \textbf{78.59} \\
    \textit{Our \ConvVAE{} + RNVP $p(z)$, $K=50$} & 80.09 \\
    \textit{Our \ResnetVAE{} + RNVP $p(z)$, $K=50$} & 79.84 \\
    \midrule
\multicolumn{2}{l}{\textit{Models with PixelCNN decoder}} \\    VLAE$^\ThreeOrMoreZMark$[\citenum{chenVariationalLossyAutoencoder2016}] & 79.03 \\
    PixelHVAE + VampPrior$^\TwoZMark$~[\citenum{tomczakVAEVampPrior2018}] & 79.78 \\
    \textit{Our \PixelVAE{} + RNVP $p(z)$, $K=50$} & \textbf{79.01} \\
  \bottomrule
\end{tabularx}
        
	\end{minipage}\hfill
	\begin{minipage}[t]{.48\linewidth}
		\flushright
		\caption{Test NLL on \MNIST{}.  ``$^\TwoZMark$'' and ``$^\ThreeOrMoreZMark$'' has the same meaning as \cref{tbl:static-mnist-nll-comparison}.}\label{tbl:mnist-nll-comparison}
		\small
		
\begin{tabularx}{\linewidth}{Xr}
  \toprule
  Model & NLL \\
    \midrule
\multicolumn{2}{l}{\textit{Models without PixelCNN decoder}} \\    ConvHVAE + Lars prior$^\TwoZMark$~[\citenum{bauerResampledPriorsVariational2018}] & 80.30 \\
    ConvHVAE + VampPrior$^\TwoZMark$~[\citenum{tomczakVAEVampPrior2018}] & 79.75 \\
    VAE + IAF$^\ThreeOrMoreZMark$~[\citenum{kingmaImprovedVariationalInference2016}] & 79.10 \\
    BIVA$^\ThreeOrMoreZMark$~[\citenum{maaloeBIVAVeryDeep2019}] & \textbf{78.41} \\
    \textit{Our \ConvVAE{} + RNVP $p(z)$, $K=50$} & 78.61 \\
    \textit{Our \ResnetVAE{} + RNVP $p(z)$, $K=50$} & 78.49 \\
    \midrule
\multicolumn{2}{l}{\textit{Models with PixelCNN decoder}} \\    VLAE$^\ThreeOrMoreZMark$~[\citenum{chenVariationalLossyAutoencoder2016}] & 78.53 \\
    PixelVAE$^\TwoZMark$~[\citenum{gulrajaniPixelvaeLatentVariable2016}] & 79.02 \\
    PixelHVAE + VampPrior$^\TwoZMark$~[\citenum{tomczakVAEVampPrior2018}] & 78.45 \\
    \textit{Our \PixelVAE{} + RNVP $p(z)$, $K=50$} & \textbf{78.12} \\
  \bottomrule
\end{tabularx}
        
	\end{minipage}
\end{table}

\begin{figure}[t]
	\centering
	\begin{minipage}[b]{.42\textwidth}
		\centering
		\includegraphics[width=0.98\linewidth]{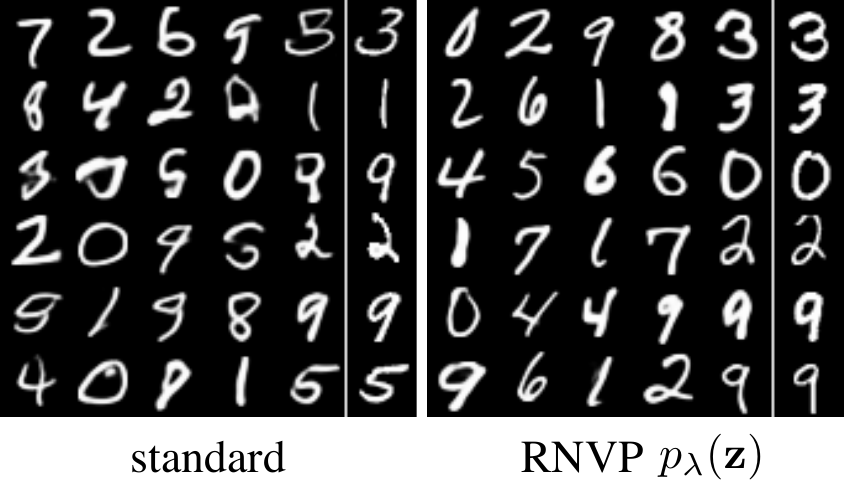}
		\caption{Sample means from $p_{\lambda}(\vv{z})$ of \ResnetVAE{} with: (left) unit Gaussian prior; (right) \RealNVP{} prior. \TheLastColumnSentence{6x6}}\label{fig:main-samples}
	\end{minipage}\hfill
	\begin{minipage}[b]{.54\textwidth}
		\centering
		\includegraphics[width=0.98\linewidth]{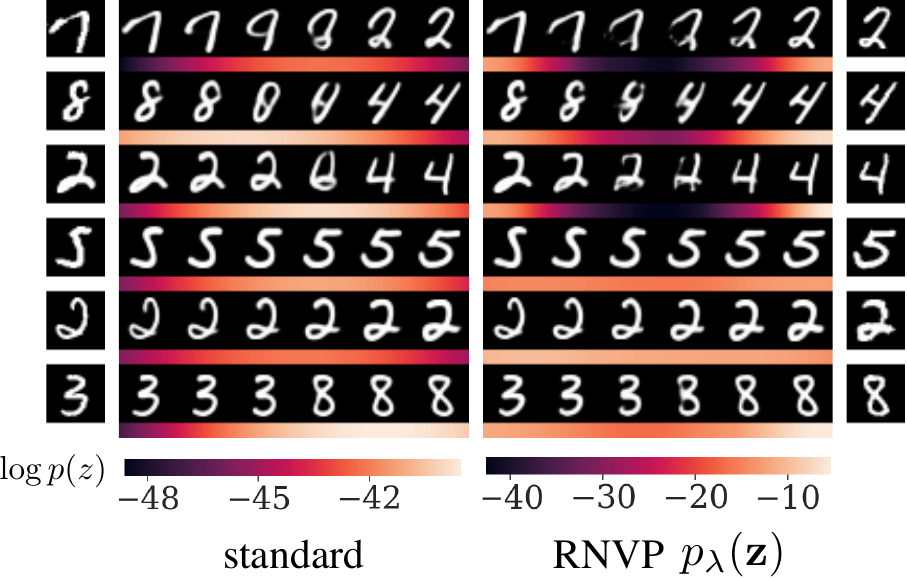}
		\caption{Interpolations of $\vv{z}$ from \ResnetVAE{}, between the centers of $q_{\phi}(\vv{z}|\vv{x})$ of two training points, and heatmaps of $\log p_{\lambda}(\vv{z})$.  The left-most and right-most columns are the original training points.}\label{fig:z-interpolation}
	\end{minipage}

\end{figure}

\begin{table}[t]
	\centering
	\caption{Average test \ELBOTerm{}, \textit{reconstruction loss} (``\textit{recons}''), \KLTerm{} (``\textit{kl}''), and \KLzxTerm{} (``$kl_{z|x}$'') of \ResnetVAE{} with different priors.}\label{tbl:elbo-recons-kl-ResnetVAE}
	\small
	\bgroup\setlength{\tabcolsep}{.65em}
    \begin{tabular}{ccccccccc}
    \toprule
& \multicolumn{4}{c}{\tablehr{standard}}& \multicolumn{4}{c}{\tablehr{RNVP \(p(z)\)}} \\
\cmidrule(lr){2-5}\cmidrule(lr){6-9}    \tablehr{Datasets} & \textit{lb}& \textit{recons}& \textit{kl}& $kl_{z|x}$ & \textit{lb}& \textit{recons}& \textit{kl}& $kl_{z|x}$ \\
    \midrule
        \StaticMNIST{}
 & 
-87.61 & 
-60.09 & 
\textbf{27.52} & 
4.67 & 
\textbf{-82.85} & 
\textbf{-54.32} & 
28.54 & 
\textbf{2.87}        \\
        \MNIST{}
 & 
-84.62 & 
-58.70 & 
\textbf{25.92} & 
3.55 & 
\textbf{-80.34} & 
\textbf{-53.64} & 
26.70 & 
\textbf{1.76}        \\
        \FashionMNIST{}
 & 
-228.91 & 
-208.94 & 
\textbf{19.96} & 
2.74 & 
\textbf{-225.97} & 
\textbf{-204.66} & 
21.31 & 
\textbf{1.88}        \\
        \Omniglot{}
 & 
-104.87 & 
-66.98 & 
\textbf{37.89} & 
7.88 & 
\textbf{-99.60} & 
\textbf{-61.21} & 
38.39 & 
\textbf{5.99}        \\
    \bottomrule
    \end{tabular}
            \egroup
\end{table}

\begin{table}[t]
	\begin{minipage}{.36\linewidth}
		\centering
		\caption{Avg. number of \textit{active units} of \ResnetVAE{} with different priors.}\label{tbl:active-units-ResnetVAE} 
		\small
		\bgroup\setlength{\tabcolsep}{.32em}
\begin{tabular}{ccc}
  \toprule
  & \multicolumn{2}{c}{\tablehr{\ResnetVAE}} \\
  \cmidrule(lr){2-3}
  \tablehr{Datasets}  & standard  & \parbox[c]{3em}{RNVP \\ \centering \(p(z)\)}  \\
  \midrule
    \StaticMNIST{} 
 &
30 &
\textbf{40}    \\
    \MNIST{} 
 &
25.3 &
\textbf{40}    \\
    \FashionMNIST{} 
 &
27 &
\textbf{64}    \\
    \Omniglot{} 
 &
59.3 &
\textbf{64}    \\
  \bottomrule
\end{tabular}
        \egroup
	\end{minipage}\hfill%
	\begin{minipage}{.62\linewidth}
		\centering
		\caption{Avg. \textit{reconstruction loss}, \KLTerm{} (``\textit{kl}'') and \textit{active units} (``\textit{au}'') of \ResnetVAE{} with \textit{iteratively trained} and \textit{post-hoc} trained \RealNVP{} priors.}\label{tbl:recons-kl-au-post-train-ResnetVAE}
		\small
		\bgroup\setlength{\tabcolsep}{.45em}
\begin{tabular}{ccccccc}
  \toprule
     & \multicolumn{3}{c}{iterative} & \multicolumn{3}{c}{post-hoc} \\
    \cmidrule(r){2-4} \cmidrule(r){5-7}
    Datasets &  \textit{recons}      &  \textit{kl}   &   \textit{au}
    &  \textit{recons}      &  \textit{kl}   &   \textit{au} \\
    \midrule
    StaticMNIST 
 &
        \textbf{-58.0}
 &
        26.4
 &
        \textbf{38.7}
 &
        -60.1
 &
        \textbf{25.3}
 &
        30
    \\
    DynamicMNIST 
 &
        \textbf{-57.2}
 &
        25.1
 &
        \textbf{40}
 &
        -58.7
 &
        \textbf{24.7}
 &
        25.3
    \\
    FashionMNIST 
 &
        \textbf{-207.8}
 &
        19.4
 &
        \textbf{64}
 &
        -208.9
 &
        \textbf{19.0}
 &
        27
    \\
    Omniglot 
 &
        \textbf{-63.3}
 &
        37.9
 &
        \textbf{64}
 &
        -67.0
 &
        \textbf{35.8}
 &
        59.3
    \\
  \bottomrule
\end{tabular}
        \egroup
	\end{minipage}
\end{table}

\begin{figure}[t]
	\centering
	\includegraphics[width=\linewidth]{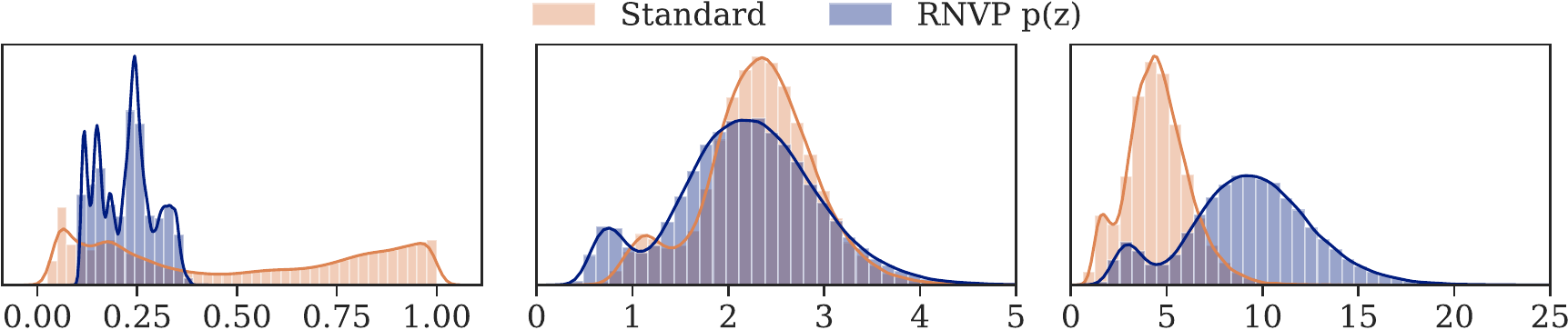}
	\caption{Histograms of: (left) per-dimensional stds of $q_{\phi}(\vv{z}|\vv{x})$; (middle) distances between closest pairs of $q_{\phi}(\vv{z}|\vv{x})$; and (right) \textit{normalized distances}.  See \cref{sec:appendix-closest-pairs-formulation} for formulation.}\label{fig:distance-std}
\end{figure}

%\begin{figure}[t]
%	\centering
%	\includegraphics[width=.6\linewidth]{figures/interpolation}
%	\caption{Interpolations of $\vv{z}$ from \ResnetVAE{}, between the centers of $q_{\phi}(\vv{z}|\vv{x})$ of two training points, and heatmaps of $\log p_{\lambda}(\vv{z})$.  The left-most and right-most columns are the original training points.}\label{fig:z-interpolation}
%\end{figure}

%%
\subsection{Setup}

\paragraph{Datasets} We use four datasets in our experiments: statically and dynamically binarized MNIST~\citep{larochelleNeuralAutoregressiveDistribution2011,salakhutdinovQuantitativeAnalysisDeep2008}, FashionMNIST~\citep{xiaoFashionMNISTNovelImage2017} and Omniglot~\citep{lakeHumanlevelConceptLearning2015}.
Details of these datasets can be found in \cref{sec:appendix-datasets}.

\myparagraph{Models} We perform systematically controlled experiments, using the following VAE variants: (1) \textbf{\DenseVAE{}}, whose encoder and decoder are composed of dense layers; (2) \textbf{\ConvVAE{}}, with convolutional layers; (3) \textbf{\ResnetVAE{}}, with ResNet layers~\citep{zagoruykoWideResidualNetworks2016}; and (4) \textbf{\PixelVAE{}}~\citep{gulrajaniPixelvaeLatentVariable2016}, with several PixelCNN layers on top of the \ResnetVAE{} decoder.
For \RealNVP{} priors and posteriors, we use $K$ blocks of invertible mappings ($K$ is called \textit{flow depth} hereafter), while each block contains an \textit{invertible dense}, a dense \textit{coupling layer}, and an \textit{actnorm}~\citep{dinhDensityEstimationUsing2016,kingmaGlowGenerativeFlow2018}.
More details 
%about model architectures 
can be found in \cref{sec:appendix-network-architectures}.

%For short, we shall denote a particular variant of VAE with Real NVP prior by \priorX{X}, \eg{}, \priorX{\ResnetVAE}

\paragraph{Training and evaluation} Unless specified, all experiments are repeated for 3 times, and the metric means are reported.
We use Adam~\citep{kingmaAdamMethodStochastic2014} and adopt warm up (KL annealing)~\citep{bowmanGeneratingSentencesContinuous2016} to train all models.
We perform early-stopping using negative log-likelihood (NLL) on validation set, to prevent over-fitting on StaticMNIST and on all datasets with \PixelVAE{}.
For evaluation, we use 1,000 samples to estimate NLL and other metrics on test set, unless specified.
More details can be found in \cref{sec:appendix-training-evaluation}.

\subsection{Quantitative results}

\cref{tbl:main-experiments-nll-mean-only} shows the NLLs of \DenseVAE{}, \ResnetVAE{} and \PixelVAE{} with flow depth $K=20$, where larger $K$ are not thoroughly tested due to limited computational resources.
\ConvVAE{} can be found in \cref{tbl:main-experiments-nll-with-std} in the Appendix, which has similar trends as \ResnetVAE{}; the standard deviations can also be found in the Appendix. 
We can see that \RealNVP{} prior consistently outperforms standard VAE and \RealNVP{} posterior in test NLL, with as large improvement as about 2 nats on \ResnetVAE{}, and even larger improvement on \DenseVAE{}.
The improvement is not so significant on \PixelVAE{}, which is not surprising because \PixelVAE{} encodes less information in the latent variable~\citep{gulrajaniPixelvaeLatentVariable2016}.

\cref{tbl:flow-depth-nll-mean-only} shows the NLLs of \ResnetVAE{} with different flow depth $K$.
Even $K=1$, \RealNVP{} prior can improve NLLs by about 1 nat.
There is no over-fitting for $K$ up to 50. %in our results.
%, although being a common concern of \citet{tomczakVAEVampPrior2018,bauerResampledPriorsVariational2018}.
However, we do not claim learning the prior will not cause over-fitting; this only suggests \RealNVP{} prior does not over-fit easily.
%In fact, \PixelVAE{} with \RealNVP{} prior encounters slight over-fitting on StaticMNIST when $K = 50$ (see \cref{tbl:main-experiments-nll-mean-only,tbl:static-mnist-nll-comparison}).

\cref{tbl:dual-rnvp-nll} shows that %the NLLs of \ResnetVAE{} with both \RealNVP{} prior and posterior.
using both \RealNVP{} prior and posterior shows no significant advantage over using \RealNVP{} prior only.
This, in conjunction with the results of \RealNVP{} prior and posterior from \cref{tbl:main-experiments-nll-mean-only}, highlights that learning the prior is crucial for good NLLs, supporting our statements in \cref{proposition-1}.
% This indicates our \ResnetVAE{} with \RealNVP{} prior might have reached its limit using the current architecture on the tested datasets.

\cref{tbl:post-train-nll} shows the NLLs of \textit{iterative training} and \textit{post-hoc training} with \ResnetVAE{}.
Although still not comparable to \textit{joint training}, both methods can bring large improvement in NLLs over standard VAE.
Also, \textit{iterative training} even further outperforms \textit{post-hoc training} by a large margin.
%Since the major difference among these training strategies is whether $q_{\phi}(\vv{z}|\vv{x})$ and $p_{\theta}(\vv{x}|\vv{z})$ are sufficiently trained with improved $p_{\lambda}(\vv{z})$, and vice versa, the results strongly support our analysis based on \cref{proposition-2}.
%the major difference among these training strategies is whether $q_{\phi}(\vv{z}|\vv{x})$ and $p_{\theta}(\vv{x}|\vv{z})$ are sufficiently trained with the learned $p_{\lambda}(\vv{x})$, which is in turn determined by the gradually improved $q_{\phi}(\vv{z}|\vv{x})$ and $p_{\theta}(\vv{x}|\vv{z})$.
%, which indicates that $p_{\lambda}(\vv{z})$ being trained to fit $q_{\phi}(\vv{z})$ might not be the only effect of \textit{iterative training}.

In \cref{tbl:static-mnist-nll-comparison,tbl:mnist-nll-comparison}, we compare \ResnetVAE{} and \PixelVAE{} with \RealNVP{} prior to other approaches on StaticMNIST and MNIST.
The results on Omniglot and FashionMNIST have a similar trend, and can be found in \cref{tbl:appendix-omniglot-nll-comparison,tbl:appendix-fashion_mnist-nll-comparison}.
All models except ours used at least 2 latent variables.
Our \ResnetVAE{} with \RealNVP{} prior, $K=50$ is second only to BIVA among all models without PixelCNN decoder, and ranks the first among all models with PixelCNN decoder.
On MNIST, the NLL of our model is very close to BIVA, while the latter used 6 latent variables and very complicated architecture.
Although BIVA has a much lower NLL on StaticMNIST, in contrast to our paper, the BIVA paper~\citep{maaloeBIVAVeryDeep2019} did not report using validation data for early-stopping,  indicating the gap %between our model and BIVA
should mainly be attributed to having fewer training data.
Meanwhile, our \ConvVAE{} with \RealNVP{} prior, $K=50$ has lower test NLL than ConvHVAE with \textit{Lars prior} and \textit{VampPrior}.
Since \ConvVAE{} is undoubtedly a simpler architecture than ConvHVAE (which has 2 latent variables), it is likely that our improvement comes from the \RealNVP{} prior rather than the different architecture.
\textbf{\cref{tbl:static-mnist-nll-comparison,tbl:mnist-nll-comparison} show that using \RealNVP{} prior with just one latent variable, it is possible to achieve NLLs comparable to very deep state-of-the-art VAE (BIVA), ourperforming many previous works (including works on priors, and works of complicated hierarchical VAE equipped with rich posteriors like VAE + IAF).
%This discovery sheds light on more simple architectures for VAE.
This discovery shows that simple VAE architectures with learned prior and a small number of latent variables is a promising direction. 
}

\subsection{Qualitative results}

\cref{fig:main-samples} samples images from \ResnetVAE{} trained with different methods.
%We can see that, 
Compared to standard \ResnetVAE{}, \ResnetVAE{} with \RealNVP{} prior produces fewer digits that are hard to interpret.
%and much fewer blurry characters, 
\TheLastColumnSentence{6x6}
However, there are differences between the last two columns, indicating our model is not just memorizing the training data.
More samples are in \cref{sec:appendix-qualitative-results}.

%%
%\subsection{Improved reconstruction loss is the main cause of improved NLLs with learned prior}\label{sec:influence-of-learning-the-prior}
\subsection{Improved reconstruction loss and other experimental results with learned prior}\label{sec:influence-of-learning-the-prior}

%We have demonstrated the learned prior can improve test NLLs by a large margin.
%In this section, we will reveal that in our experiments, the major cause of the large improvement in test NLLs is likely to be the improved reconstruction loss, rather than reduction in KL divergence, and we shall analyze why this could happen when learning the prior.
In this section, we will show the improved reconstruction loss and other experimental results with learned \RealNVP{} prior, which supports \cref{proposition-2}.
%For short, we use \KLTerm{} to denote $\KLD[q_{\phi}(\vv{z}|\vv{x})\|p_{\lambda}(\vv{z})]$, \KLzxTerm{} to denote $\KLD[q_{\phi}(\vv{z}|\vv{x})\|p_{\theta}(\vv{z}|\vv{x})]$, \KLzTerm{} to denote $\KLD[q_{\phi}(\vv{z})\|p_{\theta}(\vv{z})]$, and \ReconsTerm{} to denote $\E_{p^{\star}(\vv{x})}\,\E_{q_{\phi}(\vv{z}|\vv{x})}[\log p_{\theta}(\vv{x}|\vv{z})]$ (the expected reconstruction loss \wrt{} data distribution $p^{\star}(\vv{x})$).

%
%\paragraph{Better \ELBOTerm{} and \ReconsTerm{}, but larger \KLTerm{}}
\paragraph{Better reconstruction loss, but larger KL divergence}

In \cref{tbl:elbo-recons-kl-ResnetVAE}, \ELBOTerm{} and \ReconsTerm{} ($\E_{p^{\star}(\vv{x})}\,\E_{q_{\phi}(\vv{z}|\vv{x})}\myleft[\log p_{\theta}(\vv{x}|\vv{z}\myright]$) of \ResnetVAE{} with \RealNVP{} prior are substantially higher than standard \ResnetVAE{}, just as the trend of test log-likelihood (LL) in \cref{tbl:main-experiments-nll-mean-only}.
Metrics of other models are in \cref{tbl:kl-recons-2-DenseVAE,tbl:kl-recons-2-ConvVAE,tbl:kl-recons-2-ResnetVAE,tbl:kl-recons-2-PixelVAE}.

On the contrary, \KLTerm{} are larger, while \KLzxTerm{} are smaller.
Since $q_{\phi}(\vv{z}) = \int q_{\phi}(\vv{z}|\vv{x})\,p^{\star}(\vv{x})\,\dd{\vv{z}}$ and $p_{\lambda}(\vv{z}) = \int p_{\theta}(\vv{x}|\vv{z})\,p_{\theta}(\vv{x})\,\dd{\vv{x}}$, the fact that \RealNVP{} prior has both better test LL (\ie{}, $p_{\theta}(\vv{x})$ is closer to $p^{\star}(\vv{x})$) and lower \KLzxTerm{} should suggest $q_{\phi}(\vv{z})$ is closer to $p_{\lambda}(\vv{z})$, hence lower \KLzTerm{}.
And since $\KLTermM = \KLzTermM + \MIzxTermM$ (\eqnref{eqn:elbo-2}), this should suggest a larger \MIzx{}, \ie{}, \textit{mutual information}.
All these facts are consistent with our analysis based on \cref{proposition-2}.
Note that, under suitable conditions, \ReconsTerm{} and \KLTerm{} can happen to be both smaller (\textit{e.g.}, the results of \DenseVAE{} on StaticMNIST in \cref{tbl:kl-recons-2-DenseVAE}).

\paragraph{Smaller standard deviation of Gaussian posterior with \RealNVP{} prior}

In \cref{fig:distance-std}, we plot the histograms of per-dimensional stds of $q_{\phi}(\vv{z}|\vv{x})$, as well as the distances and \textit{normalized distances} (which is roughly distance/std) between each closest pair of $q_{\phi}(\vv{z}|\vv{x})$.
Detailed formulations can be found in \cref{sec:appendix-closest-pairs-formulation}.
The std of \RealNVP{} prior are substantially smaller, while the \textit{normalized distances} are larger.
Larger \textit{normalized distances} indicate less density of $q_{\phi}(\vv{z}|\vv{x})$ to be overlapping, hence better \textit{reconstruction loss} according to \eqnref{eqn:appendix-bernoulli-px-given-z-extermum}.
This fact is a direct evidence of \cref{proposition-2}.

\paragraph{More active units}

\cref{tbl:active-units-ResnetVAE} counts the \textit{active units}~\citep{burdaImportanceWeightedAutoencoders2015} of \ResnetVAE{} with different priors, which quantifies the number of latent dimensions used for encoding information from input data.
%Since a larger \textit{reconstruction loss} typically means larger \textit{mutual information}, learned prior may also increase the number of \textit{active units}.
\RealNVP{} prior can cause all units to be active, which is in sharp contrast to standard VAE.
It has long been a problem of VAE that the number of \textit{active units} is small, often attributed to the over-regularization of the unit Gaussian prior~\citep{hoffmanElboSurgeryAnother2016,tomczakVAEVampPrior2018}.
Learning the prior can be an effective cure for this problem.
\paragraph{Iterative training can lead to increased active units and improved reconstruction loss}

\cref{tbl:recons-kl-au-post-train-ResnetVAE} shows that, compared to \textit{post-hoc training}, \textit{iterative training} can lead to increased number of \textit{active units} and larger \ReconsTerm{}, but larger \KLTerm{}.
\cref{proposition-2} suggests that a larger \textit{reconstruction loss} can only be obtained with improved $p_{\lambda}(\vv{z})$.
However, the prior is in turn determined by $q_{\phi}(\vv{z}|\vv{x})$ and $p_{\theta}(\vv{x}|\vv{z})$ according to \cref{eqn:elbo-2}.
Thus, it is important to alternate \textit{between} training $p_{\theta}(\vv{x}|\vv{z})$ \& $q_{\phi}(\vv{z}|\vv{x})$ \textit{and} training $p_{\lambda}(\vv{z})$.
This is why \textit{iterative training} can result in larger \textit{reconstruction loss} than \textit{post-hoc training}.
%In fact, \textit{joint training} induces even better \textit{training loss} than (\cref{tbl:elbo-recons-kl-ResnetVAE}), since $q_{\phi}(\vv{z}|\vv{x})$, $p_{\theta}(\vv{x}|\vv{z})$ and $p_{\lambda}(\vv{z})$ are all trained together.

%
%In conclusion, learning the prior can lead to improved reconstruction loss.
%Better reconstruction loss often indicates less ``blurry'' samples, which has long been a critical problem of VAE, and seems inevitable with fixed priors, since it seems that we can always get to similar conclusions as \eqnref{eqn:gaussian-px-given-z-extrema}.
%This highlights the necessity of learning the prior in practice.

%%
\subsection{Learned prior on interpolated z and low posterior samples}

The standard deviations of $q_{\phi}(\vv{z}|\vv{x})$ are not reduced equally.
Instead, they are reduced according to the dissimilarity between neighbors.
This can result in a fruitful $p_{\lambda}(\vv{z})$, learned to score the interpolations of $\vv{z}$ between the centers of $q_{\phi}(\vv{z}|\vv{x}^{(i)})$ and $q_{\phi}(\vv{z}|\vv{x}^{(j)})$ of two training points $\vv{x}^{(i)}$ and $\vv{x}^{(j)}$. See \cref{fig:z-interpolation}.
%we plot the interpolations of $\vv{z}$, and respective heatmap of $\log p_{\lambda}(\vv{z})$ for the interpolated $\vv{z}$.
\RealNVP{} prior learns to give low likelihoods to hard-to-interpret interpolated samples (the first three rows), in contrast to unit Gaussian prior.
However, for good quality interpolations (the last three rows), \RealNVP{} prior grants high likelihoods.
In contrast, the unit Gaussian prior assigns high likelihoods to all interpolations, even when the samples are hard to interpret.

\citet{roscaDistributionMatchingVariational2018} observed that part of the samples from unit Gaussian prior can have low $q_{\phi}(\vv{z})$ and low visual quality, and they name these samples the ``\textit{low posterior samples}''.
Unlike this paper, they tried various approaches to match $q_{\phi}(\vv{z})$ to unit Gaussian $p_{\lambda}(\vv{z})$, including adversarial training methods, but still found \textit{low posterior samples} scattering across the whole prior.
Learning the prior can avoid having high $p_{\lambda}(\vv{z})$ on low posterior samples (see \cref{sec:appendix-learned-prior-on-low-posterior-samples}).
Since we have analyzed why $q_{\phi}(\vv{z})$ cannot match $p_{\lambda}(\vv{z})$, we suggest to adopt learned prior as a cheap solution to this problem.

\section{Related work}
Learned priors, as a natural choice for the conditional priors of intermediate variables,  have long been unintentionally used in hierarchical VAEs~\citep{rezendeStochasticBackpropagationApproximate2014,sonderbyLadderVariationalAutoencoders2016,kingmaImprovedVariationalInference2016,maaloeBIVAVeryDeep2019}.
A few works were proposed to enrich the priors of VAE, \eg{}, Gaussian mixture priors~\citep{nalisnickApproximateInferenceDeep2016,dilokthanakulDeepUnsupervisedClustering2016}, Bayesian non-parametric priors~\citep{nalisnickStickBreakingVariationalAutoencoders2016,goyalNonparametricVariationalAutoencoders2017}, and auto-regressive priors~\citep{gulrajaniPixelvaeLatentVariable2016,chenVariationalLossyAutoencoder2016}, without the awareness of its relationship with the \AggregatedPosteriorI{}, until the analysis made by \citet{hoffmanElboSurgeryAnother2016}.
Since then, several attempts have been made in matching the prior to \AggregatedPosteriorI{}, by using Real NVP~\citep{huangLearnableExplicitDensity2017}, variational mixture of posteriors~\citep{tomczakVAEVampPrior2018}, learned accept/reject sampling~\citep{bauerResampledPriorsVariational2018}, and kernel density trick~\citep{takahashiVariationalAutoencoderImplicit2018}.
However, none of these works proved the necessity of learning the prior, nor did they recognize the improved reconstruction loss induced by learned prior.
Furthermore, they did not show that learned prior with just one latent variable can achieve comparable results to those of many deep hierarchical VAEs.

The trade-off between reconstruction loss and KL divergence was also discussed by \citet{rezendeTamingVAEs2018}, but instead of relieving the resistance from the prior, they proposed to convert the reconstruction loss into an optimization constraint, so as to trade for better reconstruction at the cost of larger KL (and ELBO).
Meanwhile, \citet{roscaDistributionMatchingVariational2018} demonstrated failed attempts in matching \AggregatedPosteriorI{} to a fixed prior with expressive posteriors, and observed \textit{low posterior samples} problem.
We provide analysis on why their attempts failed, prove the necessity of learning the prior, and show the \textit{low posterior samples} can also be avoided by learned prior.

\section{Conclusion}
	In this paper, for the first time in the literature, we proved the necessity and effectiveness of learning the prior in VAE when aggregated posterior does not match unit Gaussian prior, analyzed why this situation may happen, and proposed a hypothesis that learning the prior may improve reconstruction loss, all of which are supported by our extensive experiment results. Using learned Real NVP prior with just one latent variable in VAE, we managed to achieve test NLLs comparable to very deep state-of-the-art hierarchical VAE, outperforming many previous works of complex hierarchical VAEs equipped with rich priors/posteriors. Furthermore, we demonstrated that the learned prior can avoid assigning high likelihoods to low-quality interpolations on the latent space and to the recently discovered low posterior samples.

		We believe this paper is an important step towards simple VAE architectures with learned prior and a small number of latent variables, which potentially can be more scalable to large datasets than those complex VAE architectures.

\FloatBarrier

%\subsubsection*{Acknowledgments}
%
%Use unnumbered third level headings for the acknowledgments. All acknowledgments
%go at the end of the paper. Do not include acknowledgments in the anonymized
%submission, only in the final paper.

\newpage
\printbibliography

%\bibliographystyle{abbrvnat}
%\bibliography{src/references.bib}

%%
\newpage

\appendix

\section{Proof details}

\setcounter{figure}{0}
\renewcommand{\thefigure}{A.\arabic{figure}}
\renewcommand{\theHfigure}{AH.\thefigure}

\setcounter{table}{0}
\renewcommand{\thetable}{A.\arabic{table}}
\renewcommand{\theHtable}{AH.\thetable}

\setcounter{lemma}{0}
\renewcommand{\thelemma}{A.\arabic{lemma}}
\renewcommand{\theHlemma}{AH.\thelemma}

\setcounter{equation}{0}
\renewcommand{\theequation}{A.\arabic{equation}}
\renewcommand{\theHequation}{AH.\theequation}

\subsection{Proof for \cref{proposition-1}}
\label{sec:appendix-proposition-1-proof}
\newcommand{\Lvae}{\mathcal{L}_{\mathrm{VAE}}(\theta, \phi)}
\newcommand{\Lprior}{\mathcal{L}(\lambda, \theta, \phi)}
\newcommand{\Lpriorvae}{\mathcal{L}(\lambda_0, \theta, \phi)}

The training objective for VAE with flow prior $p_{\lambda}(\vv{z}) = p_{\mathcal{N}}(f_{\lambda}(\vv{z}))\AbsDet{\partial f_{\lambda} / \partial \vv{z}}$ is:
\begin{equation*}
	\Lprior = \E_{p^{\star}(\vv{x})}\,\E_{q_\phi(\vv{z}|\vv{x})} \left[\log p_\theta(\vv{x}|\vv{z})-\log q_\phi(\vv{z}|\vv{x})+\log\AbsDet{\frac{\partial f_\lambda}{ \partial \vv{z}}} + \log p_\mathcal{N}\left(f_\lambda(\vv{z})\right) \right]	
\end{equation*}
where $p_{\mathcal{N}}(\cdot)$ denotes unit Gaussian distribution $\mathcal{N}(\vv{0},\vv{I})$.
Meanwhile, the training objective of a standard VAE with prior $p_{\mathcal{N}}(\vv{z})$ is:
\begin{equation*}
	\Lvae = \E_{p^{\star}(\vv{x})}\,\E_{q_\phi(\vv{z}|\vv{x})} \left[\log p_\theta(\vv{x}|\vv{z})-\log q_\phi(\vv{z}|\vv{x})+\log p_\mathcal{N}(\vv{z}) \right]
\end{equation*}

To prove \cref{proposition-1}, we first introduce the following lemmas:

\begin{lemma}\label{lemma-A1}
	\begin{equation*}
		\max_{\theta,\phi} \Lvae \leq \max_{\lambda,\theta,\phi} \Lprior
	\end{equation*}
\end{lemma}
\begin{proof}
	Let $\lambda_0$ be the set of parameters satisfying $f_{\lambda_0} = \IdentityMapping$ (identity mapping), then $\frac{\partial f_{\lambda_0}}{\partial \vv{z}} = \vv{I}$, and:
	\begin{align*}
		\Lpriorvae &= \E_{p^{\star}(\vv{x})}\,\E_{q_\phi(\vv{z}|\vv{x})} \left[\log p_\theta(\vv{x}|\vv{z})-\log q_\phi(\vv{z}|\vv{x})+\log\AbsDet{\frac{\partial f_{\lambda_0}}{ \partial \vv{z}}} + \log p_\mathcal{N}\left(f_{\lambda_0}(\vv{z})\right) \right] \\
		&= \E_{p^{\star}(\vv{x})}\,\E_{q_\phi(\vv{z}|\vv{x})} \left[\log p_\theta(\vv{x}|\vv{z})-\log q_\phi(\vv{z}|\vv{x})+\log p_\mathcal{N}(\vv{z}) \right]
	\end{align*}
 	which means $\Lvae = \Lpriorvae$, thus we have:
 	\begin{equation*}
 		\Lvae \leq \max_{\lambda} \Lprior	
 	\end{equation*}
	Since for all $\theta$ and $\phi$, the above inequality always holds, we have:
	\begin{equation*}
		\max_{\theta,\phi} \Lvae \leq \max_{\lambda,\theta,\phi} \Lprior
	\end{equation*}
\end{proof}

\begin{lemma}\label{lemma-A2}
	For all $\theta,\phi$,
	\begin{equation*}
		\Lvae = \max_{\lambda} \Lprior
	\end{equation*}
	only if $q_{\phi}(\vv{z}) = p_{\mathcal{N}}(\vv{z})$.
\end{lemma}
\begin{proof}
	As \cref{lemma-A1} has proved, $\Lvae = \Lpriorvae$, thus we only need to prove $\lambda_0$ is not the optimal solution of $\max_{\lambda} \Lprior$.
	
	We start by introducing a non-parameterized continuous function $f(\vv{z})$, and rewrite $\Lprior$ as a functional on $f$:
	\begin{equation*}
		\mathcal{L}[f] = \E_{p^{\star}(\vv{x})}\,\E_{q_\phi(\vv{z}|\vv{x})} \left[\log p_\theta(\vv{x}|\vv{z})-\log q_\phi(\vv{z}|\vv{x})+\log\AbsDet{\frac{\partial f}{ \partial \vv{z}}} + \log p_\mathcal{N}\myleft(f(\vv{z})\myright) \right]
	\end{equation*}
	According to \citet{hornikApproximationCapabilitiesMultilayer1991}, neural networks can represent any continuous function defined on $\mathbb{R}^n$.
	If we can find a continuous differentiable function $f(\vv{z})$, which will give $\mathcal{L}[f] > \Lvae$, then there must exist a neural network derived $f_{\lambda}(\vv{z}) = f(\vv{z})$, \st{} $\Lprior > \Lpriorvae$.
	Because of this, although we can only apply calculus of variations on continuous differentiable functions, it is sufficient to prove \cref{lemma-A2} with this method.
	We write $\mathcal{L}[f]$ into the form of Euler's equation:
%	Since the current architectures of neural networks are all continuous mappings, and according to \citet{hornikApproximationCapabilitiesMultilayer1991} neural networks can represent any continuous function defined on $\mathbb{R}^n$, we can use calculus of variations to find the optimal $f(\vv{z})$, and then obtain the optimal parameter $\lambda$ such that $f_{\lambda}(\vv{z}) = f(\vv{z})$ immediately.  We further write $\mathcal{L}[f]$ into the form of Euler's equation:
	\begin{equation*}
		\mathcal{L}[f] = \int F\myleft(\vv{z},f,\frac{\partial f}{\partial \vv{z}}\myright)\,\dd{\vv{z}}
	\end{equation*}
	where
	\begin{align*}
		&F\myleft(\vv{z},f,\frac{\partial f}{\partial \vv{z}}\myright) \\
		=& \int p^{\star}(\vv{x})\,q_{\phi}(\vv{z}|\vv{x})\left[ \log p_{\theta}(\vv{x}|\vv{z}) - \log q_{\phi}(\vv{z}|\vv{x}) + \log\AbsDet{\frac{\partial f}{\partial \vv{z}}} + \log p_{\mathcal{N}}\myleft(f(\vv{z})\myright) \right]\,\dd{\vv{x}} \\
		=& \int p^{\star}(\vv{x})\,q_{\phi}(\vv{z}|\vv{x})\left[ \log p_{\theta}(\vv{x}|\vv{z}) - \log q_{\phi}(\vv{z}|\vv{x}) + \log\AbsDet{\frac{\partial f}{\partial \vv{z}}} + \sum_k \log p_{\mathcal{N}}\myleft(f^i(\vv{z})\myright) \right]\,\dd{\vv{x}}
	\end{align*}
	Here we use $f^i(\vv{z})$ to denote the $i$-th dimension of the output of $f(\vv{z})$, while $x_j$ and $z_j$ to denote the $j$-th dimension of $\vv{x}$ and $\vv{z}$.
	For $\log\AbsDet{\partial f/\partial \vv{z}}$, we can further expand it \wrt{} the $k$-th row:
	\begin{equation*}
		\log\AbsDet{\frac{\partial f}{\partial \vv{z}}} = 
			\log \left[ \sum_j \frac{\partial f^i}{\partial z_j} (-1)^{i+j} M_{ij} \right]	
	\end{equation*}
	where $M_{ij}$ is the $(i,j)$ minor of the Jacobian matrix $\partial f/\partial \vv{z}$.

	Assume we have $\Lvae = \max_{\lambda} \Lprior$.  That is, $\mathcal{L}[f]$ attains its extremum at $f = f_{\lambda_0} = \IdentityMapping$, or:
	\begin{equation}
		\frac{\delta L}{\delta f} = 0 \label{eqn:lemma-A2-L-partial}	
	\end{equation}
	According to Euler's equation~\citep[page 14 and 35]{gelfandCalculusVariations2000}, the necessary condition for \eqnref{eqn:lemma-A2-L-partial} is:
	\begin{equation}
		\frac{\partial F}{\partial f^i} - \sum_j \frac{\partial}{\partial z_j} \frac{\partial F}{\partial (\partial_{z_j} f^i)} = 0 \label{eqn:lemma-A2-eular-equation}
	\end{equation}
	for all $i$. Note we use $\partial_{z_j} f^i$ to denote $\frac{\partial f^i}{\partial z_j}$.

	Consider the term $\frac{\partial F}{\partial f^i}$, we have:
	\begin{equation*}
		\frac{\partial F}{\partial f^i} = \int p^{\star}(\vv{x})\,q_{\phi}(\vv{z}|\vv{x}) \cdot \frac{1}{p_{\mathcal{N}}\myleft(f^i(\vv{z})\myright)}	\cdot \frac{\partial p_{\mathcal{N}}\myleft(f^i(\vv{z})\myright)}{\partial f^i}\,\dd{\vv{x}}
	\end{equation*}
	where $p_{\mathcal{N}}\myleft(f^i(\vv{z})\myright) = \frac{1}{\sqrt{2\pi}}\exp\myleft[-\frac{\myleft(f^i(\vv{z})\myright]^2}{2}\myright]$, thus we have:
	\begin{equation*}
		\frac{\partial p_{\mathcal{N}}\myleft(f^i(\vv{z})\myright)}{\partial f^i} = \frac{1}{\sqrt{2\pi}} \exp\myleft[-\frac{\myleft(f^i(\vv{z})\myright)^2}{2}\myright] \cdot \left(-f^i(\vv{z})\right) = -p_{\mathcal{N}}\myleft(f^i(\vv{z})\myright)\,f^i(\vv{z})
	\end{equation*}
	Therefore,
	\begin{equation}
		\frac{\partial F}{\partial f^i} = \int p^{\star}(\vv{x})\,q_{\phi}(\vv{z}|\vv{x}) \cdot \frac{1}{p_{\mathcal{N}}\myleft(f^i(\vv{z})\myright)}	\cdot \left[-p_{\mathcal{N}}\myleft(f^i(\vv{z})\myright)\,f^i(\vv{z})\right] \,\dd{\vv{x}} = \int -p^{\star}(\vv{x})\,q_{\phi}(\vv{z}|\vv{x})\,f^i(\vv{z})\,\dd{\vv{x}} \label{eqn:lemma-A2-term-1}
	\end{equation}

	Consider the other term $\sum_j \frac{\partial}{\partial z_j} \frac{\partial F}{\partial (\partial_{z_j} f^i)}$,
	\begin{align*}
		\frac{\partial F}{\partial (\partial_{z_j} f^i)} &= \int \frac{\partial \myleft( p^{\star}(\vv{x})\,q_{\phi}(\vv{z}|\vv{x})\,\log\AbsDet{\partial f/\partial \vv{z}} \myright)}{\partial (\partial_{z_j} f^i)}\,\dd{\vv{x}} \\
		&= \int p^{\star}(\vv{x})\,q_{\phi}(\vv{z}|\vv{x})\,\frac{\partial\myleft( \log \left[ \sum_k \frac{\partial f^i}{\partial z_k} (-1)^{i+k} M_{ik} \right] \myright)}{\partial (\partial_{z_j} f^i)}\,\dd{\vv{x}} \\
		&= \int p^{\star}(\vv{x})\,q_{\phi}(\vv{z}|\vv{x})\cdot\frac{1}{\AbsDet{\partial f / \partial \vv{z}}}\cdot(-1)^{i+j} M_{ij}\,\dd{\vv{x}}
	\end{align*}
	Since $f = \IdentityMapping$, we have $\partial f/\partial \vv{z}=\vv{I}$, thus $\partial f/\partial \vv{z}$ is independent on $z_j$, and:
	\begin{align}
		\frac{\partial}{\partial z_j} \frac{\partial F}{\partial (\partial_{z_j} f^i)} &= \int \frac{\partial q_{\phi}(\vv{z}|\vv{x})}{\partial z_j}\cdot p^{\star}(\vv{x})\cdot\frac{1}{\AbsDet{\partial f / \partial \vv{z}}}\cdot(-1)^{i+j} M_{ij}\,\dd{\vv{x}} \label{eqn:lemma-A2-term-2} \\
		M_{ij} &= \delta_{ij} = \begin{cases}
			1, & \text{if}\;\; i=j \\
			0, & \text{otherwise}
		\end{cases} \label{eqn:lemma-A2-Mij}
	\end{align}

	Substitute \cref{eqn:lemma-A2-term-1}, \eqref{eqn:lemma-A2-term-2} and \eqref{eqn:lemma-A2-Mij} into \eqnref{eqn:lemma-A2-eular-equation}, we have:
	\begin{equation*}
		\int -p^{\star}(\vv{x})\,q_{\phi}(\vv{z}|\vv{x})\,f^i(\vv{z}) - \sum_j \left[ \frac{\partial q_{\phi}(\vv{z}|\vv{x})}{\partial z_j}\cdot p^{\star}(\vv{x})\cdot(-1)^{i+j}\,\delta_{ij} \right]\,\dd{\vv{x}} = 0
	\end{equation*}
%	Since $f = \IdentityMapping$, we have $\partial f/\partial \vv{z}=\vv{I}$ and $\AbsDet{\partial f / \partial \vv{z}} = 1$.  Also, $M_{ij} = \delta_{ij}$, where:
%	\begin{equation*}
%		\delta_{ij} = \begin{cases}
%			1, & \text{if}\;\; i=j \\
%			0, & \text{otherwise}
%		\end{cases}
%	\end{equation*}
	We then have:
	\begin{align*}
		& \int -p^{\star}(\vv{x})\,q_{\phi}(\vv{z}|\vv{x})\,z_i - \frac{\partial q_{\phi}(\vv{z}|\vv{x})}{\partial z_i}\cdot p^{\star}(\vv{x})\cdot(-1)^{i+i} \delta_{ii} \,\dd{\vv{x}} = 0 \\
		\implies & \int -p^{\star}(\vv{x})\,q_{\phi}(\vv{z}|\vv{x})\,z_i - \frac{\partial q_{\phi}(\vv{z}|\vv{x})}{\partial z_i}\,p^{\star}(\vv{x}) \,\dd{\vv{x}} = 0 \\
		\implies & z_i \int -p^{\star}(\vv{x})\,q_{\phi}(\vv{z}|\vv{x})\,\dd{\vv{x}} = \int \frac{\partial q_{\phi}(\vv{z}|\vv{x})}{\partial z_i}\,p^{\star}(\vv{x})\,\dd{\vv{x}} \\
		\implies & z_i \int -p^{\star}(\vv{x})\,q_{\phi}(\vv{z}|\vv{x})\,\dd{\vv{x}} = \frac{\partial}{\partial z_i} \int q_{\phi}(\vv{z}|\vv{x})\,p^{\star}(\vv{x})\,\dd{\vv{x}} \\
		\implies & -z_i \cdot q_{\phi}(\vv{z}) = \frac{\partial}{\partial z_i}\, q_{\phi}(\vv{z})
	\end{align*}
%	Solve the above differential equation, we have:
%	\begin{equation*}
%		q_{\phi}(\vv{z}) = C \exp\myleft[-\int \vv{z} \dd{\vv{z}}\myright] = C \exp\myleft[-\frac{\left\|\vv{z}\right\|^2}{2}\myright]
%	\end{equation*}
%	Since $q_{\phi}(\vv{z})$ is a probability distribution, we have $q_{\phi}(\vv{z}) = p_{\mathcal{N}}(\vv{z})$.
    Let $q_{\phi}(\mathbf{z}) = q_{\phi}(z_1|z_2,\dots,z_K) \cdot q_{\phi}(z_2,\dots,z_K)$, where $K$ is the number of dimensions of $\mathbf{z}$.  We shall first solve the differential equation \wrt{} $z_1$:  
    \begin{align*}
    & -z_1 \cdot q_{\phi}(\mathbf{z}) = \frac{\partial}{\partial z_1}\,q_{\phi}(\mathbf{z}) \numberthis\label{eqn:lemma-A2-pde-1} \\
    \implies & -z_1 \cdot q_{\phi}(z_1|z_2,\dots,z_K)\cdot q_{\phi}(z_2,\dots,z_K) = \frac{\partial}{\partial z_1}\,q_{\phi}(z_1|z_2,\dots,z_K)\cdot q_{\phi}(z_2,\dots,z_K) \\
    \implies & -z_1 \cdot q_{\phi}(z_1|z_2,\dots,z_K)\cdot q_{\phi}(z_2,\dots,z_K) = q_{\phi}(z_2,\dots,z_K)\cdot \frac{\partial}{\partial z_1}\,q_{\phi}(z_1|z_2,\dots,z_K) \\
    \implies & -z_1 \cdot q_{\phi}(z_1|z_2,\dots,z_K) = \frac{\partial}{\partial z_1}\,q_{\phi}(z_1|z_2,\dots,z_K) \\
    \implies & -z_1\,\partial z_1 = \frac{1}{q_{\phi}(z_1|z_2,\dots,z_K)}\,\partial q_{\phi}(z_1|z_2,\dots,z_K) \\
    \implies & \int -z_1\,\partial z_1 = \int \frac{1}{q_{\phi}(z_1|z_2,\dots,z_K)}\,\partial q_{\phi}(z_1|z_2,\dots,z_K) \\
    \implies & -\frac{1}{2} z_1^2 + C(z_2,\dots,z_K) = \log q_{\phi}(z_1|z_2,\dots,z_K) \\
    \implies & \exp\myleft(-\frac{1}{2} z_1^2\myright)\cdot\exp\myleft(C(z_2,\dots,z_K)\myright) = q_{\phi}(z_1|z_2,\dots,z_K)
    \end{align*}
    Since $q_{\phi}(z_1|z_2,\dots,z_K)$ is a probability distribution, we have:
    \begin{equation*}
    	\exp\myleft(C(z_2,\dots,z_K)\myright) = \frac{1}{\int \exp\myleft(-\frac{1}{2} z_1^2\myright) \,\mathrm{d}z_1} = \frac{1}{\sqrt{2\pi}}
    \end{equation*}
	thus we have:
    \begin{equation*}
    q_{\phi}(\mathbf{z}) = q_{\phi}(z_1|z_2,\dots,z_K) \cdot q_{\phi}(z_2,\dots,z_K) = \frac{1}{\sqrt{2\pi}} \exp\myleft(-\frac{1}{2} z_1^2\myright)\cdot q_{\phi}(z_2,\dots,z_K)
    \end{equation*}
    We then solve the differential equation \wrt{} $z_2$:
    \begin{align*}
    & -z_2 \cdot q_{\phi}(\mathbf{z}) = \frac{\partial}{\partial z_2}\,q_{\phi}(\mathbf{z}) \\
    \implies & -z_2 \cdot \frac{1}{\sqrt{2\pi}} \exp\myleft(-\frac{1}{2} z_1^2\myright) \cdot q_{\phi}(z_2,\dots,z_K) = \frac{\partial}{\partial z_2} \frac{1}{\sqrt{2\pi}} \exp\myleft(-\frac{1}{2} z_1^2\myright) \cdot q_{\phi}(z_2,\dots,z_K) \\
    \implies & -z_2 \cdot \frac{1}{\sqrt{2\pi}} \exp\myleft(-\frac{1}{2} z_1^2\myright) \cdot q_{\phi}(z_2,\dots,z_K) = \frac{1}{\sqrt{2\pi}} \exp\myleft(-\frac{1}{2} z_1^2\myright) \cdot\frac{\partial}{\partial z_2} q_{\phi}(z_2,\dots,z_K) \\
    \implies & -z_2 \cdot q_{\phi}(z_2,\dots,z_K) = \frac{\partial}{\partial z_2} q_{\phi}(z_2,\dots,z_K) \numberthis\label{eqn:lemma-A2-pde-2}
    \end{align*}
    If we let $\mathbf{z}’ = z_2,\dots,z_K$, the form of \cref{eqn:lemma-A2-pde-2} is now exactly identical with \cref{eqn:lemma-A2-pde-1}.  Use the same method, we can solve the equation \wrt{} $z_2$, and further \wrt{} $z_3, \dots, z_K$.  Finally ,we can get the solution:
    \begin{equation*}
        q_{\phi}(\mathbf{z}) = \frac{1}{\left(\sqrt{2\pi}\right)^K}\prod_{i=1}^K \exp\myleft(-\frac{1}{2} z_i^2\myright)
    \end{equation*}
    which is $K$-dimensional unit Gaussian, \ie{}, $q_{\phi}(\vv{z}) = p_{\mathcal{N}}(\vv{z})$.
\end{proof}

%
% We shall further prove there always exist a better $f_{\lambda}$ than $f_{\lambda_0}$ when $q_{\phi}(\vv{z}) \neq p_{\lambda}(\vv{z})$, as follows:
\begin{lemma}\label{lemma-A3}
	For all $\theta, \phi$, if $q_{\phi}(\vv{z}) \neq p_{\mathcal{N}}(\vv{z})$, $\exists f_{\lambda} \neq f_{\lambda_0}$, \st{} $\mathcal{L}(\lambda,\theta,\phi) > \mathcal{L}(\lambda_0,\theta,\phi)$.
\end{lemma}
\begin{proof}
	If $q_{\phi}(\vv{z}) \neq p_{\lambda}(\vv{z})$, then according to \eqnref{eqn:lemma-A2-L-partial}, we have:
	\begin{equation*}
		\left.\frac{\delta L}{\delta f} \right|_{f = f_{\lambda_0}} \neq 0
	\end{equation*}
 	Then there must exist $f_{\lambda}$ in the neighborhood of $f_{\lambda_0}$, such that $\Lprior > \Lpriorvae$.
\end{proof}

Finally, we get to the proof for \cref{proposition-1}:
\begin{proposition*}
	For VAE with flow prior $p_{\lambda}(\vv{z}) = p_{\mathcal{N}}(f_{\lambda}(\vv{z}))\AbsDet{\partial f_{\lambda}/\partial \vv{z}}$, where $p_{\mathcal{N}}(f_{\lambda}(\vv{z})) = \mathcal{N}(\vv{0},\vv{I})$, if $q_{\phi}(\vv{z}) \neq \mathcal{N}(\vv{0},\vv{I})$, then its training objective (\ELBOTerm{}):
\begin{equation*}
	\mathcal{L}(\lambda,\theta,\phi) = \E_{p^{\star}(\vv{x})}\,\E_{q_{\phi}(\vv{z}|\vv{x})}\left[ \log p_{\theta}(\vv{x}|\vv{z}) + \log p_{\mathcal{N}}(f_{\lambda}(\vv{z})) + \log  \AbsDet{\frac{\partial f_{\lambda}(\vv{z})}{\partial\vv{z}}} - \log q_{\phi}(\vv{z}|\vv{x}) \right]
\end{equation*}
can reach its extremum only if $f_{\lambda} \neq f_{\lambda_0}$, where $f_{\lambda_0} = \IdentityMapping$ is the identity mapping.
Also, $\forall \theta, \phi$, if $q_{\phi}(\vv{z}) \neq \mathcal{N}(\vv{0},\vv{I})$, then there exist $f_{\lambda} \neq f_{\lambda_0}$, \st{} $\mathcal{L}(\lambda,\theta,\phi) > \mathcal{L}(\lambda_0,\theta,\phi)$.

\end{proposition*}
\begin{proof}
	\textbf{Necessity}\hspace{0.5em}
	According to \cref{lemma-A2}, $\Lvae = \max_{\lambda} \Lprior$ implies $q_{\phi}(\vv{z}) = p_{\mathcal{N}}(\vv{z})$.
	Take \cref{lemma-A1} into consideration, it means if $q_{\phi}(\vv{z}) \neq p_{\mathcal{N}}(\vv{z})$, then $\forall \theta, \phi$, $\Lvae = \Lpriorvae < \max_{\lambda} \Lprior$.
	Hence, $f_{\lambda} \neq f_{\lambda_0}$ is the necessary condition for $\Lprior$ to reach its extremum if $q_{\phi}(\vv{z}) \neq p_{\mathcal{N}}(\vv{z})$.

	\textbf{Effectiveness}\hspace{0.5em}
	According to \cref{lemma-A3}, there always exist a $f_{\lambda} \neq f_{\lambda_0}$ when $q_{\phi}(\vv{z}) \neq p_{\lambda}(\vv{z})$, \st{} $\Lprior > \Lpriorvae = \Lvae$.
\end{proof}

\subsection{Proof for \cref{proposition-2}}
\label{sec:appendix-bernoulli-px-given-z-extermum}
%\begin{proposition*}
%	For finite number of discrete training data, \ie{}, the training distribution $p^{\star}(\vv{x}) = \frac{1}{N} \sum_{i=1}^N \delta(\vv{x}-\vv{x}^{(i)})$, if we let $p_{\theta}(\vv{x}|\vv{z}) = \mathrm{Bernoulli}(\vv{\mu}_{\theta}(\vv{z}))$, where the Bernoulli mean $\vv{\mu}_{\theta}(\vv{z})$ is produced by the decoder, $x_k$ is the $k$-th dimension of $\vv{x}$ and $\mu^{k}_{\theta}(\vv{z})$ is the $k$-th dimension of $\vv{\mu}_{\theta}(\vv{z})$, then for any particular encoder $q_{\phi}(\vv{z}|\vv{x})$ and prior $p_{\lambda}(\vv{z})$, the optimal decoder $\vv{\mu}_{\theta}(\vv{z})$ is:
%	\begin{equation*}
%		\vv{\mu}_{\theta}(\vv{z}) = \sum_i w_i(\vv{z})\,\vv{x}^{(i)}, \quad \text{where} \;\,
%		w_i(\vv{z}) = \frac{q_{\phi}(\vv{z}|\vv{x}^{(i)})}{\sum_j q_{\phi}(\vv{z}|\vv{x}^{(j)})} \;\, \text{and}\;\, \sum_i w_i(\vv{z}) = 1
%	\end{equation*}
%\end{proposition*}

\begin{proof}
	To apply calculus of variations, we need to substitute the parameterized, bounded $\vv{\mu}_{\theta}(\vv{z})$ with a non-parameterized, unbounded mapping.
	Since $0 < \mu_{\theta}^k(\vv{z}) < 1$, and $\vv{\mu}_{\theta}(\vv{z})$ is produced by neural network, which ensures $\vv{\mu}_{\theta}(\vv{z})$ is a continuous mapping, then $\forall \theta$, there exists unbounded $\vv{t}(\vv{z})$, \st{}
%	We first substitute the bounded $\vv{\mu}_{\theta}(\vv{z})$ with an unbounded $\vv{t}(\vv{z})$:
	\begin{align*}
		\mu^k_{\theta}(\vv{z}) &= \frac{\Exp{t^k(\vv{z})}}{1 + \Exp{t^k(\vv{z})}} \\
		t^k(\vv{z}) &= \log \mu^k_{\theta}(\vv{z}) - \log \myleft( 1 - \mu^k_{\theta}(\vv{z}) \myright)
	\end{align*}
	and for all continuous mapping $\vv{t}(\vv{z})$, there also exists $\vv{\mu}_{\theta}(\vv{z})$, satisfying the above equations.
	In fact, this substitution is also adopted in the actual implementation of our models.

	The probability of $p_{\theta}(\vv{x}|\vv{z})$ is given by:
	\begin{equation*}
		p_{\theta}(\vv{x}|\vv{z}) = \mathrm{Bernoulli}(\vv{\mu}_{\theta}(\vv{z})) = \prod_{k} \left(\mu^k_{\theta}(\vv{z})\right)^{x_k}\left(1-\mu^k_{\theta}(\vv{z})\right)^{(1-x_k)}	
	\end{equation*}
	Then we have:
	\begin{align*}
		\log p_{\theta}(\vv{x}|\vv{z}) &= \sum_k \Big\{x_k \log \mu^k_{\theta}(\vv{z}) + (1-x_k) \log \left(1-\mu^k_{\theta}(\vv{z})\right)\Big\} \\
		&= \sum_k \Big\{ x_k \,t^k(\vv{z}) - x_k \log\left[1 + \Exp{t^k(\vv{z})}\right] - (1-x_k) \log\left[1 + \Exp{t^k(\vv{z})}\right] \Big\} \\
		&= \sum_k \Big\{ x_k \,t^k(\vv{z}) - \log\left[1 + \Exp{t^k(\vv{z})}\right] \Big\}
	\end{align*}
	The training objective $\mathcal{L}$ can be then formulated as a functional on $\vv{t}(\vv{z})$:
	\begin{align*}
		\mathcal{L}[\vv{t}] &= \E_{p^{\star}(\vv{x})}\,\E_{q_{\phi}(\vv{z}|\vv{x})}\left[\log p_{\theta}(\vv{x}|\vv{z}) + \log \frac{p_{\theta}(\vv{z})}{q_{\phi}(\vv{z}|\vv{x})} \right] \\
		&= \iint p^{\star}(\vv{x}) \, q_{\phi}(\vv{z}|\vv{x}) \left(\log p_{\theta}(\vv{x}|\vv{z}) + \log \frac{p_{\theta}(\vv{z})}{q_{\phi}(\vv{z}|\vv{x})} \right) \dd{\vv{z}} \, \dd{\vv{x}} \\
		&= \int F(\vv{z},\vv{t}) \,\dd{\vv{z}}
	\end{align*}
	where $F(\vv{z},\vv{t})$ is:
	\begin{equation*}
		F(\vv{z},\vv{t}) = \int p^{\star}(\vv{x}) \, q_{\phi}(\vv{z}|\vv{x}) \left[ \sum_k \Big\{ x_k \,t^k(\vv{z}) - \log\left[1 + \Exp{t^k(\vv{z})}\right] \Big\} + \log \frac{p_{\theta}(\vv{z})}{q_{\phi}(\vv{z}|\vv{x})} \right] \dd{\vv{x}}
	\end{equation*}
	According to Euler's equation~\citep[page 14 and 35]{gelfandCalculusVariations2000}, the necessary condition for $\mathcal{L}[\vv{t}]$ to have an extremum for a given $\vv{t}(\vv{z})$ is that, $\vv{t}(\vv{z})$ satisfies $\partial F / \partial t^k = 0, \forall k$.  Thus we have:
	\begin{align*}
		\frac{\partial F}{\partial t^k} = 0 
		\implies & \int p^{\star}(\vv{x}) \, q_{\phi}(\vv{z}|\vv{x}) \left[ x_k - \frac{\Exp{t^k(\vv{z})}}{1 + \Exp{t^k(\vv{z})}} \right] \,\dd{\vv{x}} = 0 \\
		\implies & \int p^{\star}(\vv{x}) \, q_{\phi}(\vv{z}|\vv{x}) \left[ x_k - \mu_{\theta}^k(\vv{z}) \right] \,\dd{\vv{x}} = 0 \\
		\implies & \sum_i q_{\phi}(\vv{z}|\vv{x}^{(i)})\left[x_k^{(i)} - \mu^k_{\theta}(\vv{z})\right] = 0 \\
		\implies & \sum_i q_{\phi}(\vv{z}|\vv{x}^{(i)}) \, x_k^{(i)} = \left(\sum_i q_{\phi}(\vv{z}|\vv{x}^{(i)})\right) \mu_{\theta}^k(\vv{z}) \\
		\implies & \mu_{\theta}^k(\vv{z}) = \frac{\sum_i q_{\phi}(\vv{z}|\vv{x}^{(i)})\,x_k^{(i)}}{\sum_j q_{\phi}(\vv{z}|\vv{x}^{(j)})}
	\end{align*}
	That is to say, $\vv{\mu}_{\theta}(\vv{z}) = \frac{\sum_i q_{\phi}(\vv{z}|\vv{x}^{(i)})\,\vv{x}^{(i)}}{\sum_j q_{\phi}(\vv{z}|\vv{x}^{(j)})} = \sum_i w_i(\vv{z})\,\vv{x}^{(i)}$.
\end{proof}

\citet{rezendeTamingVAEs2018} has proved that when $p_{\theta}(\vv{x}|\vv{z}) = \mathcal{N}(\vv{\mu}_{\theta}(\vv{z}),\sigma^2 \vv{I})$, where $\sigma$ is a global fixed constant, the optimal decoder $\vv{\mu}_{\theta}(\vv{z}) = \sum_i w_i(\vv{z})\,\vv{x}^{(i)}$, which is exactly the same as our conclusion.
\citet{roscaDistributionMatchingVariational2018} has proved that the gradient of $\mathrm{Bernoulli}(\vv{\mu}_{\theta}(\vv{z}))$ is the same as $\mathcal{N}(\vv{\mu}_{\theta}(\vv{z}),\sigma^2 \vv{I})$ when $\sigma = 1$, but they did not calculate out the optimal decoder.
We push forward both these works.

\subsection{Trade-off between reconstruction loss and mutual information}
\label{sec:appendix-reconstruction-loss-and-mutual-information}
%Ideally, if $p_{\theta}(\mathbf{x})=p^{\star}(\mathbf{x})$, and $q_{\phi}(\mathbf{z}|\mathbf{x})=p_{\theta}(\mathbf{z}|\mathbf{x})$, we can ignore the subscripts $\lambda,\theta,\phi$, and have:
%\begin{align*}
%	\mathbb{I}[Z;X] &= H[X] + H[Z] - H[X,Z] \\
%		&= H[X] + H[Z] - \left(H[X|Z] + H[Z]\right) \\
%		&= H[X] - H[X|Z] \\
%		&= H[X] - \left(-\mathbb{E}_{p(\mathbf{x})}\mathbb{E}_{p(\mathbf{z}|\mathbf{x})}\left[\log p(\mathbf{x}|\mathbf{z})\right]\right) \\
%		&= H[X] + \text{Reconstruction Loss}
%\end{align*}
%where $H[X]$ is the entropy of data, which is a constant.  Thus a larger reconstruction loss should typically indicate a larger mutual information.

To show there is a trade-off between \textit{reconstruction loss} and \textit{mutual information}, we first assume the \textit{mutual information} $\MI_{\phi}[Z;X]$ reaches its optimum value.
Since
\begin{equation*}
	\MI_{\phi}[Z;X] = \iint q_{\phi}(\vv{z},\vv{x}) \log \frac{q_{\phi}(\vv{z},\vv{x})}{q_{\phi}(\vv{z})\,p^{\star}(\vv{x})} \,\dd{\vv{z}}\,\dd{\vv{x}} = \KLD[q_{\phi}(\vv{z},\vv{x})\|q_{\phi}(\vv{z})\,p^{\star}(\vv{x})]
\end{equation*}
we can see that $\MI_{\phi}[Z;X]$ reaches its minimum value 0 if and only if $q_{\phi}(\vv{z},\vv{x}) = q_{\phi}(\vv{z})\,p^{\star}(\vv{x})$.
This means $q_{\phi}(\vv{z}|\vv{x}) = q_{\phi}(\vv{z})$ for all $\vv{x}$ and $\vv{z}$, since $q_{\phi}(\vv{z},\vv{x}) = q_{\phi}(\vv{z}|\vv{x})\,p^{\star}(\vv{x})$.
According to \cref{proposition-2}, we then have:
\begin{equation*}
	\vv{\mu}_{\theta}(\vv{z}) = \frac{\sum_i q_{\phi}(\vv{z}|\vv{x}^{(i)})\,\vv{x}^{(i)}}{\sum_j q_{\phi}(\vv{z}|\vv{x}^{(j)})} = \frac{\sum_i q_{\phi}(\vv{z})\,\vv{x}^{(i)}}{\sum_j q_{\phi}(\vv{z})} = \frac{1}{N} \sum_i \vv{x}^{(i)}
\end{equation*}
this means the decoder $\vv{\mu}_{\theta}(\vv{z})$ will produce the same reconstruction output for all $\vv{z}$, which is the average of all training data, hence causing  a poor \textit{reconstruction loss}.

The fact that \textit{mutual information} can reach its optimum value only when having a poor \textit{reconstruction loss} indicates there is trade-off between \textit{reconstruction loss} and \textit{mutual information}.

%
%\subsection{Proof for \eqnref{eqn:gaussian-px-given-z-extrema}}
%\label{sec:appendix-gaussian-px-given-z-extrema}
%\input{src/appendix/gaussian_px_given_z_extrema}

%%
\section{Experimental details}

\setcounter{figure}{0}
\renewcommand{\thefigure}{B.\arabic{figure}}
\renewcommand{\theHfigure}{BH.\thefigure}

\setcounter{table}{0}
\renewcommand{\thetable}{B.\arabic{table}}
\renewcommand{\theHtable}{BH.\thetable}

\setcounter{algorithm}{0}
\renewcommand{\thealgorithm}{B.\arabic{algorithm}}
\renewcommand{\theHalgorithm}{BH.\arabic{algorithm}}

\setcounter{lemma}{0}
\renewcommand{\thelemma}{B.\arabic{lemma}}
\renewcommand{\theHlemma}{BH.\thelemma}

\setcounter{equation}{0}
\renewcommand{\theequation}{B.\arabic{equation}}
\renewcommand{\theHequation}{BH.\theequation}

\subsection{Datasets}
\label{sec:appendix-datasets}

\paragraph{MNIST}
MNIST is a 28x28 grayscale image dataset of hand-written digits, with 60,000 data points for training and 10,000 for testing.
When validation is required for early-stopping, we randomly split the training data into 50,000 for training and 10,000 for validation.

Since we use Bernoulli $p_{\theta}(\vv{x}|\vv{z})$ to model these images in VAE, we binarize these images by the method in \citep{salakhutdinovQuantitativeAnalysisDeep2008}: each pixel value is randomly set to 1 in proportion to its pixel intensity.
The training and validation images are re-binarized at each epoch.
However, the test images are binarized beforehand for all experiments.
We binarize each test image 10 times, and use all these 10 binarized data points in evaluation.
This method results in a 10 times larger test set, but we believe this can help us to obtain a more objective evaluation result.

\paragraph{StaticMNIST}
StaticMNIST~\citep{larochelleNeuralAutoregressiveDistribution2011} is a pre-binarized MNIST image dataset, with the original 60,000 training data already splitted into 50,000 for training and 10,000 for validation.
We always use validation set for early-stopping on StaticMNIST.
Meanwhile, since StaticMNIST has already been binarized, the test set is used as-is without 10x enlargement.

\paragraph{FashionMNIST}
FashionMNIST~\citep{xiaoFashionMNISTNovelImage2017} is a recently proposed image dataset of grayscale fashion products, with the same specification as MNIST.
We thus use the same training-validation split and the same binarization method just as MNIST.

\paragraph{Omniglot}
Omniglot~\citep{lakeHumanlevelConceptLearning2015} is a 28x28 grayscale image dataset of hand-written characters.
We use the preprocessed data from \citep{burdaImportanceWeightedAutoencoders2015}, with 24,345 data points for training and 8,070 for testing.
When validation is required, we randomly split the training data into 20,345 for training and 4,000 for validation.
We use dynamic binarization on Omniglot just as MNIST.

\FloatBarrier

\subsection{Network architectures}
\label{sec:appendix-network-architectures}
\paragraph{Notations}

In order to describe the detailed architecture of our models, we will introduce auxiliary functions to denote network components.
A function $h_{\phi}(\vv{x})$ should denote a sub-network in $q_{\phi}(\vv{z}|\vv{x})$, with a subset of $\phi$ as its own learnable parameters.
For example, if we write $q_{\phi}(\vv{z}|\vv{x}) = q_{\phi}(\vv{z}|h_{\phi}(\vv{x})) = \mathcal{N}(\vv{\mu}_{\phi}(h_{\phi}(\vv{x})),\vv{\sigma}^2_{\phi}(h_{\phi}(\vv{x}))\,\vv{I})$, it means that the the posterior $q_{\phi}(\vv{z}|\vv{x})$ is a Gaussian, whose mean and standard deviation are derived by one shared sub-network $h_{\phi}(\vv{x})$ and two separated sub-networks $\vv{\mu}_{\phi}(\cdot)$ and $\vv{\sigma}_{\phi}(\cdot)$, respectively.

The structure of a network is described by composition of elementary neural network layers.

$\NetLinear[k]$ indicates a linear dense layer with $k$ outputs.
$\NetDense[k]$ indicates a non-linear dense layer.
$a \NetTo b$ indicates a composition of $a$ and $b$, \eg{}, $Dense[m] \NetTo Dense[n]$ indicates two successive dense layers, while the first layer has $m$ outputs and the second layer has $n$ outputs.
These two dense layers can also be abbreviated as $\NetDense[m \NetTo n]$.

$\NetConv[H \times W \times C]$ denotes a non-linear convolution layer, whose output shape is $H \times W \times C$, where $H$ is the height, $W$ is the width and $C$ is the channel size.
As abbreviation, $\NetConv[H_1 \times W_1 \times C_1 \NetTo H_2 \times W_2 \times C_2]$ denotes two successive non-linear convolution layers.
$\NetResnet[\cdot]$ denotes non-linear resnet layer(s)~\citep{zagoruykoWideResidualNetworks2016}.
All $\NetConv$ and $\NetResnet$ layers by default use 3x3 kernels, unless the kernel size is specified as subscript (\eg{}, $\NetConv_{1 \times 1}$ denotes a 1x1 convolution layer).
The strides of $\NetConv$ and $\NetResnet$ layers are automatically determined by the input and output shapes, which is 2 in most cases.
$\NetLinearConv[\cdot]$ denotes linear convolution layer(s).

$\NetDeConv[\cdot]$ and $\NetDeResnet[\cdot]$ denotes deconvolution and deconvolutional resnet layers, respectively.

$\NetPixelCNN[\cdot]$ is a PixelCNN layer proposed by~\citet{salimansPixelcnnImprovingPixelcnn2017}.
It uses resnet layers, instead of convolution layers.
Details can be found in its original paper.

$\NetFlatten$ indicates to reshape the input 3-d tensor into a vector, while $\NetUnFlatten[H \times W \times C]$ indicates to reshape the input vector into a 3-d tensor of shape $H \times W \times C$.
$\NetConcat[a,b]$ indicates to concat the output of $a$ and $b$ along the last axis.

$\NetCouplingLayer$ and $\NetActNorm$ are components of Real NVP, proposed by \citet{dinhDensityEstimationUsing2016,kingmaGlowGenerativeFlow2018}.
$\NetInvertibleDense$ is a component modified from \textit{invertible 1x1 convolution}~\citep{kingmaGlowGenerativeFlow2018}.
We shall only introduce the details of $\NetCouplingLayer$, since it contains sub-networks, while the rest two are just simple components.

\paragraph{General configurations}

All non-linear layers use \textit{leaky relu}~\citep{maasRectifierNonlinearitiesImprove2013} activation.

The observed variable $\vv{x}$ (\ie{}, the input image) is always a 3-d tensor, with shape $H \times W \times C$, whether or not the model is convolutional.
% For \DenseVAE{}, \ConvOrResnetVAE{} and \PixelVAE{}, 
The latent variable $\vv{z}$ is a vector, whose number of dimensions is chosen to be 40 on MNIST and StaticMNIST, while 64 on FashionMNIST and Omniglot.
This is because we think the latter two datasets are conceptually more complicated than the other two, thus requiring higher dimensional latent variables.

The Gaussian posterior $q_{\phi}(\vv{z}|\vv{x})$ is derived as:
\begin{align*}
	q_{\phi}(\vv{z}|\vv{x}) &= \mathcal{N}(\vv{\mu}_{\phi}(h_{\phi}(\vv{x})),\vv{\sigma}^2_{\phi}(h_{\phi}(\vv{x}))\,\vv{I})	\\
	\vv{\mu}_{\phi}(h_{\phi}(\vv{x})) &= h_{\phi}(\vv{x}) \NetTo \NetLinear[\Dim(\vv{z})] \\
	\log \vv{\sigma}_{\phi}(h_{\phi}(\vv{x})) &= h_{\phi}(\vv{x}) \NetTo \NetLinear[\Dim(\vv{z})]
\end{align*}
Note we make the network to produce $\log \vv{\sigma}_{\phi}(h_{\phi}(\vv{x}))$ instead of directly producing $\vv{\sigma}_{\phi}(h_{\phi}(\vv{x}))$.
$h_{\phi}(\vv{x})$ is the hidden layers, varying among different models.

For binarized images, we use Bernoulli conditional distribution $p_{\theta}(\vv{x}|\vv{z})$, derived as:
\begin{align*}
	p_{\theta}(\vv{x}|\vv{z}) &= \mathrm{Bernoulli}[\vv{\mu}_{\theta}(h_{\theta}(\vv{z}))] \\
	\log \frac{\vv{\mu}_{\theta}(h_{\theta}(\vv{z}))}{1-\vv{\mu}_{\theta}(h_{\theta}(\vv{z}))} &= \begin{cases}
			h_{\theta}(\vv{z}) \NetTo \NetLinear[784] \NetTo \NetUnFlatten[28 \times 28 \times 1] &\text{for \DenseVAE{}} \\
			h_{\theta}(\vv{z}) \NetTo \NetLinearConv_{1 \times 1}[28 \times 28 \times 1] &\text{otherwise}
		\end{cases}
\end{align*}
Note we make the network to produce $\log \frac{\vv{\mu}_{\theta}(h_{\theta}(\vv{z}))}{1-\vv{\mu}_{\theta}(h_{\theta}(\vv{z}))}$, the \textit{logits} of Bernoulli distribution, instead of producing the Bernoulli mean $\vv{\mu}_{\theta}(h_{\theta}(\vv{z}))$ directly.

\paragraph{\DenseVAE{}}
$h_{\phi}(\vv{x})$ and $h_{\theta}(\vv{z})$ of \DenseVAE{} are composed of dense layers, formulated as:
\begin{align*}
	h_{\phi}(\vv{x}) &= \vv{x} \NetTo \NetFlatten \NetTo \NetDense[500 \NetTo 500] \\
	h_{\theta}(\vv{z}) &= \vv{z} \NetTo \NetDense[500 \NetTo 500]
\end{align*}

\paragraph{\ConvOrResnetVAE{}}
$h_{\phi}(\vv{x})$ and $h_{\theta}(\vv{z})$ of \ConvVAE{} are composed of (de)convolutional layers, while those of \ResnetVAE{} consist of (deconvolutional) resnet layers.
We only describe the architecture of \ResnetVAE{} here.
The structure of \ConvVAE{} can be easily obtained by replacing all (deconvolutional) resnet layers with (de)convolution layers:
\begin{align*}
	h_{\phi}(\vv{x}) &= \vv{x} \NetTo \NetResnet[
		28 \times 28 \times 32 
			\NetTo 28 \times 28 \times 32 
			\NetTo 14 \times 14 \times 64 \\
			&\qquad\qquad\qquad\;\, \NetTo 14 \times 14 \times 64 
			\NetTo 7 \times 7 \times 64 
			\NetTo 7 \times 7 \times 16
		]  \\
		& \qquad \NetTo\NetFlatten \\
	h_{\theta}(\vv{z}) &= \vv{z} \NetTo \NetDense[784] \NetTo \NetUnFlatten[7 \times 7 \times 16] \\
		&\quad\;\;\;\NetTo\NetDeResnet[
			7 \times 7 \times 64
			\NetTo 14 \times 14 \times 64
			\NetTo 14 \times 14 \times 64 \\
			&\qquad\qquad\qquad\quad\;\; \NetTo 28 \times 28 \times 32
			\NetTo 28 \times 28 \times 32
		]
\end{align*}

\paragraph{\PixelVAE{}}
$h_{\phi}(\vv{x})$ of \PixelVAE{} is the exactly same as \ResnetVAE{}, while $h_{\theta}(\vv{z})$ is derived as:
\begin{align*}
	h_{\theta}(\vv{z}) &= \NetConcat[\vv{x	},\tilde{h}_{\theta}(\vv{z})] \\
		&\quad\,\NetTo \NetPixelCNN[
			28 \times 28 \times 33
			\NetTo 28 \times 28 \times 33
			\NetTo 28 \times 28 \times 33
		] \\
	\tilde{h}_{\theta}(\vv{z}) &= \vv{z} \NetTo \NetDense[784] \NetTo \NetUnFlatten[7 \times 7 \times 16] \\
		&\quad\;\;\;\NetTo\NetDeResnet[
			7 \times 7 \times 64
			\NetTo 14 \times 14 \times 64
			\NetTo 14 \times 14 \times 64 \\
			&\qquad\qquad\qquad\quad\;\; \NetTo 28 \times 28 \times 32
			\NetTo 28 \times 28 \times 32
		]
\end{align*}
As \citet{salimansPixelcnnImprovingPixelcnn2017}, we use dropout in PixelCNN layers,  with rate 0.5.

\paragraph{Real NVP}

The \RealNVP{} consists of $K$ blocks, while each block consist of an \textit{invertible dense}, a \textit{coupling layer}, and an \textit{actnorm}.
The Real NVP mapping for prior, \ie{} $f_{\lambda}(\vv{z})$, can be formulated as:
\begin{align*}
	f_{\lambda}(\vv{z}) &= \vv{z} \NetTo f_1(\vv{h}_1) \NetTo \cdots \NetTo f_K(\vv{h}_K) \\
	f_k(\vv{h}_k) &= \vv{h}_k \NetTo \NetInvertibleDense \NetTo \NetCouplingLayer \NetTo \NetActNorm \\
	\NetCouplingLayer(\vv{u}) &=
		\NetConcat\left[\vv{u}_l,\vv{u}_r\odot \mathrm{Sigmoid}\left(s_{\lambda}(h_{\lambda,k}(\vv{u}_l))\right) + t_{\lambda}(h_{\lambda,k}(\vv{u}_l))\right] \\
	h_{\lambda,k}(\vv{u}_l) &= \vv{u}_l \NetTo \NetDense[256] \\
	s_{\lambda}(h_{\lambda,k}(\vv{u}_l)) &= h_{\lambda,k}(\vv{u}_l) \NetTo \NetLinear[\Dim(\vv{u}_r)] \\
	t_{\lambda}(h_{\lambda,k}(\vv{u}_l)) &= h_{\lambda,k}(\vv{u}_l) \NetTo \NetLinear[\Dim(\vv{u}_r)]
\end{align*}
where $\vv{u}_l = \vv{u}_{0 : \lfloor\Dim(\vv{u})/2\rfloor}$ is the left half of $\vv{u}$, and $\vv{u}_r = \vv{u}_{\lfloor\Dim(\vv{u})/2\rfloor:\Dim(\vv{u})}$ is the right half.

The \RealNVP{} posterior, derived from the original Gaussian posterior $q_{\phi}(\vv{w}|\vv{x})$, is denoted as $q_{\phi,\eta}(\vv{z}|\vv{x})$, and formulated as:
\begin{align*}
	q_{\phi,\eta}(\vv{z}|\vv{x}) &= q_{\phi}(\vv{w}|\vv{x}) \left|\frac{\partial f_{\eta}(\vv{w})}{\partial \vv{w}}\right|^{-1} \numberthis\label{eqn:posterior-rnvp-qz} \\
	\vv{z} &= f_{\eta}(\vv{w})
\end{align*}
where the structure of the \RealNVP{} mapping $f_{\eta}(\vv{w})$ for posterior is exactly the same as $f_{\lambda}(\vv{z})$ for prior.
The ELBO for VAE with \RealNVP{} posterior is then simply:
\begin{align*}
	\mathcal{L}(\vv{x};\lambda,\theta,\phi,\eta) = 
		\E_{q_{\phi}(\vv{w}|\vv{x})}\bigg[& \log p_{\theta}(\vv{x}|f_{\eta}(\vv{w})) + \log p_{\lambda}(f_{\eta}(\vv{w})) \\& - \log q_{\phi}(\vv{w}|\vv{x}) + \log \AbsDet{\frac{\partial f_{\eta}(\vv{w})}{\partial\vv{w}}} \bigg]
\end{align*}

\FloatBarrier

\subsection{Additional details of training and evaluation}
\label{sec:appendix-training-evaluation}
\paragraph{General methodology}

All the mathematical expressions of expectations \wrt{} some distributions are computed by Monte Carlo integration.
For example, $\E_{q_{\phi}(\vv{z}|\vv{x})}\left[f(\vv{z},\vv{x})\right]$ is estimated by:
\begin{equation*}
	\E_{q_{\phi}(\vv{z}|\vv{x})}\left[f(\vv{z},\vv{x})\right] \approx
		\frac{1}{L} \sum_{i=1}^L f(\vv{z}^{(i)},\vv{x}), \quad \text{where} \; \vv{z}^{(i)} \; \text{is one sample from} \; q_{\phi}(\vv{z}|\vv{x})
\end{equation*}

\paragraph{Training}

We use Adam~\citep{kingmaAdamMethodStochastic2014} to train our models.
The models are trained for 2,400 epochs.
The batch size is 128 for \DenseVAE{}, \ConvVAE{} and \ResnetVAE{}, and 64 for \PixelVAE{}.
On MNIST, FashionMNIST and Omniglot, we set the learning rate to be $10^{-3}$ in the first 800 epochs, $10^{-4}$ in the next 800 epochs, and $10^{-5}$ in the last 800 epochs.
On StaticMNIST, we set the learning rate to be $10^{-4}$ in the first 1,600 epochs, and $10^{-5}$ in the last 800 epochs. 

L2 regularization with factor $10^{-4}$ is applied on weights of all non-linear hidden layers, \ie{}, kernels of non-linear dense layers and convolutional layers, in $h_{\phi}(\vv{x})$, $h_{\theta}(\vv{z})$, $f_{\lambda}(\vv{z})$ and $f_{\eta}(\vv{w})$.

The ELBO is estimated by 1 $\vv{z}$ sample for each $\vv{x}$ in training.
We adopt warm-up (KL annealing)~\citep{bowmanGeneratingSentencesContinuous2016}.
The ELBO using warm-up is formulated as:
\begin{equation*}
	\mathcal{L}(\vv{x};\lambda,\theta,\phi)
    = \E_{q_{\phi}(\vv{z}|\vv{x})}\left[ \log p_{\theta}(\vv{x}|\vv{z}) + \beta \,\left(\log p_{\lambda}(\vv{z}) - \log q_{\phi}(\vv{z}|\vv{x})\right) \right] 
\end{equation*}
$\beta$ is increased from 0.01 to 1 linearly in the first 100 epochs, and it remains 1 afterwards.
The warm-up ELBO for VAEs with Real NVP priors and posteriors can be obtained by replacing $p_{\lambda}(\vv{z})$ and $q_{\phi}(\vv{z}|\vv{x})$ of the above equation by \cref{eqn:real-nvp-px} and \cref{eqn:posterior-rnvp-qz}, respectively.

We adopt early-stopping using NLL on validation set, to prevent over-fitting on StaticMNIST and \PixelVAE{}.
The validation NLL is estimated using 100 $\vv{z}$ samples for each $\vv{x}$, every 20 epochs.

\paragraph{Training strategies for $p_{\lambda}(\vv{z})$}

We consider three training strategies for optimizing \eqnref{eqn:elbo-rnvp-pz}:

\begin{itemize}
	\item \textbf{Post-hoc training}~\citep{bauerResampledPriorsVariational2018}: $q_{\phi}(\vv{z}|\vv{x})$ and $p_{\theta}(\vv{x}|\vv{z})$ are firstly trained \wrt{} the unit Gaussian prior, then $q_{\phi}(\vv{z}|\vv{x})$ and $p_{\theta}(\vv{x}|\vv{z})$ are fixed and $p_{\lambda}(\vv{z})$ is in turn optimized.  This is the most intuitive training method according to \eqnref{eqn:elbo-2}, however, it does not work  as well as \textit{joint training} in terms of test negative log-likelihood (NLL), which is observed both in our experiments (\cref{tbl:post-train-nll}) and by \citet{bauerResampledPriorsVariational2018}.
	\item \textbf{Joint training}~\citep{tomczakVAEVampPrior2018,bauerResampledPriorsVariational2018}: $p_{\lambda}(\vv{z})$ are jointly trained along with $q_{\phi}(\vv{z}|\vv{x})$ and $p_{\theta}(\vv{x}|\vv{z})$, by directly maximizing \eqnref{eqn:elbo-rnvp-pz}.
	\item \textbf{Iterative training}: Proposed by us, we alternate \textit{between} training $p_{\theta}(\vv{x}|\vv{z})$ \& $q_{\phi}(\vv{z}|\vv{x})$ \textit{and} training $p_{\lambda}(\vv{z})$, for multiple iterations.  The first iteration to train $p_{\theta}(\vv{x}|\vv{z})$ \& $q_{\phi}(\vv{z}|\vv{x})$ should use the unit Gaussian prior.  Early-stopping should be performed during the whole process if necessary.  See \cref{algo:iterative-training} for detailed procedure of this strategy.  We adopt this method mainly for investigating why \textit{post-hoc training} does not work as well as \textit{joint training}.
\end{itemize}

To train $p_{\lambda}(\vv{z})$ with \textit{post-hoc} strategy, we start from a trained VAE, adding \RealNVP{} prior onto it, and optimizing the Real NVP $f_{\lambda}(\vv{z})$ for 3,200 epochs, with learning rate set to $10^{-3}$ in the first 1,600 epochs, and $10^{-4}$ in the final 1,600 epochs.

\begin{algorithm}[h!]
	\caption{Pseudocode for iterative training.}
	\label{algo:iterative-training}

	\begin{algorithmic}
		\State \textbf{Iteration 1a}: Train $q_{\phi}(\vv{z}|\vv{x})$ and $p_{\theta}(\vv{x}|\vv{z})$, with $p_{\lambda}(\vv{z}) = \mathcal{N}(\vv{0},\vv{I})$.
		\State \textbf{Iteration 1b}: Train $p_{\lambda}(\vv{z})$ for $2M$ epochs, with fixed $q_{\phi}(\vv{z}|\vv{x})$ and $p_{\theta}(\vv{x}|\vv{z})$.
		\For{i = 2 \dots I}
			\State \textbf{Iteration \textit{i}a}: Train $q_{\phi}(\vv{z}|\vv{x})$ and $p_{\theta}(\vv{x}|\vv{z})$ for $M$ epochs, with fixed $p_{\lambda}(\vv{z})$.
			\State \textbf{Iteration \textit{i}b}: Train $p_{\lambda}(\vv{z})$ for $M$ epochs, with fixed $q_{\phi}(\vv{z}|\vv{x})$ and $p_{\theta}(\vv{x}|\vv{z})$.
		\EndFor
	\end{algorithmic}
\end{algorithm}

\cref{algo:iterative-training} is the pseudocode of \textit{iterative training} strategy.
For Iteration 1a, all hyper-parameters are the same with training a standard VAE, where in particular, the training epoch is set to 2,400.
For Iteration 1b, the learning rate is $10^{-3}$ for the first $M$ epochs, and is $10^{-4}$ for the next $M$ epochs.
For all the next iterations, learning rate is always $10^{-4}$.
The number of iterations $I$ is chosen to be 16, and the number of epochs $M$ is chosen to be 100, for MNIST, FashionMNIST and Omniglot.
For StaticMNIST, we find it overfits after only a few iterations, thus we choose $I$ to be 4, and $M$ to be 400.
With these hyper-parameters, $q_{\phi}(\vv{z}|\vv{x})$, $p_{\theta}(\vv{x}|\vv{z})$ and $p_{\lambda}(\vv{z})$ are \textit{iteratively trained} for totally 3,200 epochs on all datasets (starting from Iteration 1b), after the pre-training step (Iteration 1a).

%We then train the model for 3,200 epochs using the following procedure:
%\begin{enumerate}
%	\item Train $p_{\lambda}(\vv{z})$ for 800 epochs, with learning rate set to $10^{-3}$ for the first 400 epochs, and $10^{-4}$ for the next 400 epochs;
%	\item Train $p_{\theta}(\vv{x}|\vv{z})$ and $q_{\phi}(\vv{z}|\vv{x})$ for 400 epochs, with learning rate set to $10^{-4}$;
%	\item Train $p_{\lambda}(\vv{z})$ for 400 epochs, with learning rate set to $10^{-4}$
%	\item Repeat (2) and (3) for another two times.
%\end{enumerate}

%
\paragraph{Regularization term for $q_{\phi}(\vv{z})$}

In order to compare some metrics (\eg{}, the \textit{active units}) of VAE with \RealNVP{} prior to those of standard VAE, we introduce an additional regularization term for $q_{\phi}(\vv{z})$, such that $q_{\phi}(\vv{z})$ of VAE using \RealNVP{} prior would have roughly zero mean and unit variance, just as a standard VAE with unit Gaussian prior.
%When the prior $p_{\lambda}(\vv{z})$ is being jointly trained along with $q_{\phi}(\vv{z}|\vv{x})$, these two distributions can expand/shrink or drift simultaneously, without changing ELBO.
%This would result in infinitely many optimum solutions, potentially causing trouble in optimization.
%We thus propose to use a regularization term to pin each dimension of the \textit{aggregated posterior} $q_{\phi}(\vv{z})$ to zero mean and a specific variance.
The regularization term for $q_{\phi}(\vv{z})$ (denoted as $\mathrm{Reg}\left[q_{\phi}(\vv{z})\right]$) and the final training objective augmented with the regularization term (denoted as $\widetilde{\mathcal{L}}(\vv{x};\lambda,\theta,\phi)$) is:
\begin{align*}
	\mathrm{Reg}\left[q_{\phi}(\vv{z})\right]
		&= \frac{1}{\Dim(\vv{z})} \sum_{k=1}^{\Dim(\vv{z})} \Big[\left(\mathrm{Mean}[z_k]\right)^2 + \left(\mathrm{Var}[z_k]-1\right)^2 \Big] \\
	\widetilde{\mathcal{L}}(\lambda,\theta,\phi) &= \E_{p^{\star}(\vv{x})}\,\mathcal{L}(\vv{x};\lambda,\theta,\phi) + \mathrm{Reg}\left[q_{\phi}(\vv{z})\right]
\end{align*}
where $z_k$ is the $k$-th dimension of $\vv{z}$, $\Dim(\vv{z})$ is the number of dimensions, $\mathrm{Mean}[z_k] = \E_{p^{\star}(\vv{x})}\,\E_{q_{\phi}(\vv{z}|\vv{x})}\left[z_k\right]$ and $\mathrm{Var}[z_k] = \E_{p^{\star}(\vv{x})}\,\E_{q_{\phi}(\vv{z}|\vv{x})}\left[\left(z_k -\mathrm{Mean}[z_k]\right)^2\right]$ are the mean and variance of each dimension.

\begin{table}[!h]
	\centering
	\caption{Avg. $\mathrm{Mean}[z_k]$ and $\mathrm{Var}[z_k]$ of regularized/un-regularized \ResnetVAE{} with \RealNVP{} prior.}\label{tbl:z_mean_var-ResnetVAE}
	\bgroup\setlength{\tabcolsep}{.75em}
    \begin{tabular}{ccccc}
    \toprule
& \multicolumn{2}{c}{\tablehr{regularized}}& \multicolumn{2}{c}{\tablehr{un-regularized}} \\
\cmidrule(lr){2-3}\cmidrule(lr){4-5}    \tablehr{Datasets} & Avg. $\mathrm{Mean}[z_k]$& Avg. $\mathrm{Var}[z_k]$ & Avg. $\mathrm{Mean}[z_k]$& Avg. $\mathrm{Var}[z_k]$ \\
    \midrule
        \StaticMNIST{}
 & 
        -0.02 $\pm$ 0.03 & 
        0.93 $\pm$ 0.01 & 
        0.05 $\pm$ 0.37 & 
        1.50 $\pm$ 0.02        \\
        \MNIST{}
 & 
        0.00 $\pm$ 0.00 & 
        0.98 $\pm$ 0.01 & 
        -0.06 $\pm$ 0.02 & 
        0.76 $\pm$ 0.04        \\
        \FashionMNIST{}
 & 
        0.00 $\pm$ 0.00 & 
        0.98 $\pm$ 0.00 & 
        0.00 $\pm$ 0.03 & 
        0.86 $\pm$ 0.07        \\
        \Omniglot{}
 & 
        0.00 $\pm$ 0.02 & 
        0.95 $\pm$ 0.00 & 
        0.00 $\pm$ 0.01 & 
        0.24 $\pm$ 0.01        \\
    \bottomrule
    \end{tabular}
            \egroup
\end{table}

\cref{tbl:z_mean_var-ResnetVAE} shows the average $\mathrm{Mean}[z_k]$ and $\mathrm{Var}[z_k]$ of \ResnetVAE{} with \RealNVP{} prior, computed on test data.
$\text{Average}\;\mathrm{Mean}[z_k]$ is defined as $\frac{1}{\Dim(\vv{z})} \sum_{k=1}^{\Dim(\vv{z})} \mathrm{Mean}[z_k]$, while $\text{average}\;\mathrm{Var}[z_k]$ is defined as $\frac{1}{\Dim(\vv{z})} \sum_{k=1}^{\Dim(\vv{z})} \mathrm{Var}[z_k]$.
The means and standard deviations of the above table is computed \wrt{} repeated experiments.
Using the regularization term for $q_{\phi}(\vv{z})$ makes the $\mathrm{Mean}[z_k]$ and $\mathrm{Var}[z_k]$ of $\vv{z}$ samples close to $\mathcal{N}(\vv{0},\vv{I})$, which is in sharp contrast with the un-regularized case.
For fair comparison in \cref{tbl:active-units-ResnetVAE,fig:distance-std}, it is crucial to have $\mathrm{Mean}[z_k]$ and $\mathrm{Var}[z_k]$ close to $\mathcal{N}(\vv{0},\vv{I})$.
Test NLLs are not reported here, because we find no significant difference between regularized and un-regularized models in terms of test NLLs.

\paragraph{Evaluation}

The \textit{negative log-likelihood} (NLL), the \textit{reconstruction loss} and the KL divergence $\E_{p^{\star}(\vv{x})}\,\KLD[q_{\phi}(\vv{z}|\vv{x})\|p_{\lambda}(\vv{z})]$ are estimated with 1000 $\vv{z}$ samples from $q_{\phi}(\vv{z}|\vv{x}^{(i)})$ for each $\vv{x}^{(i)}$ from test data:
\begin{align*}
	&\text{NLL} \approx \frac{1}{N} \sum_{i=1}^N \LogMeanExp_{j=1}^{1000}\left[ \log p_{\theta}(\vv{x}^{(i)}|\vv{z}^{(i,j)}) + \log p_{\lambda}(\vv{z}^{(i,j)}) - \log q_{\phi}(\vv{z}^{(i,j)}|\vv{x}^{(i)}) \right] \\
	&\text{Reconstruction Loss} \approx \frac{1}{1000 N} \sum_{i=1}^N \sum_{j=1}^{1000}\left[ \log p_{\theta}(\vv{x}^{(i)}|\vv{z}^{(i,j)}) \right] \\
	&\E_{p^{\star}(\vv{x})}\,\KLD[q_{\phi}(\vv{z}|\vv{x})\|p_{\lambda}(\vv{z})] \approx \frac{1}{1000 N} \sum_{i=1}^N \sum_{j=1}^{1000}\left[ \log q_{\phi}(\vv{z}^{(i,j)}|\vv{x}^{(i)}) - \log p_{\lambda}(\vv{z}^{(i,j)}) \right]
\end{align*}
where each $\vv{z}^{(i,j)}$ is one sample from $q_{\phi}(\vv{z}|\vv{x}^{(i)})$, and $\LogMeanExp_{j=1}^L \left[f(\vv{x}^{(i)},\vv{z}^{(i,j)})\right]$ is:
\begin{align*}
	\LogMeanExp_{j=1}^L\left[f(\vv{x}^{(i)},\vv{z}^{(i,j)})\right] &= 
		f_{max} + \log \frac{1}{L} \sum_{j=1}^L \left[ \exp\myleft(f(\vv{x}^{(i)},\vv{z}^{(i,j)}) - f_{max}\myright) \right] \\
		f_{max} &= \max_j f(\vv{x}^{(i)},\vv{z}^{(i,j)})
\end{align*}

\paragraph{Active units}

\textit{active units}~\citep{burdaImportanceWeightedAutoencoders2015} is defined as the number of latent dimensions whose variance is larger than 0.01.
%The active units of the $k$-th dimension is formulated as:
The variance of the $k$-th dimension is formulated as:
\begin{equation*}
	\Var_k = \Var_{p^{\star}(\vv{x})}\left[ \E_{q_{\phi}(\vv{z}|\vv{x})}\left[z_k\right] \right]
\end{equation*}
where $z_k$ is the $k$-th dimension of $\vv{z}$.
We compute $\Var_k$ on the training data, while the inner expectation $\E_{q_{\phi}(\vv{z}|\vv{x})}\left[z_k\right]$ is estimated by drawing 1,000 samples of $\vv{z}$ for each $\vv{x}$. 

%For a dimension \vv{u} of \vv{z}, its variance $Var(\vv{u})$ is calculated as:
%\begin{align*}
%	Var(\vv{u}) = Var_{\vv{x}}[E_{\vv{u}\sim q_{\phi}(\vv{u}|\vv{x})}(\vv{u})]
%\end{align*}
%The results of \textit{active units} are obtained on the training set. For each $\vv{x}^{(i)}$, $E_{\vv{u}\sim q_{\phi}(\vv{u}|\vv{x}^{(i)})}(\vv{u})$ is estimated by drawing 1000 samples of $\vv{u}$.
\FloatBarrier

\subsection[Formulation of closest pairs of q(z|x) and others]{Formulation of closest pairs of $q_{\phi}(\vv{z}|\vv{x})$ and others}
\label{sec:appendix-closest-pairs-formulation}
\paragraph{Closest pairs of $q_{\phi}(\vv{z}|\vv{x})$}

For each $\vv{x}^{(i)}$ from training data, we find the training point $\vv{x}^{(j)}$, whose posterior $q_{\phi}(\vv{z}|\vv{x}^{(j)})$ is the closest neighbor to $q_{\phi}(\vv{z}|\vv{x}^{(i)})$:
\begin{equation*}
j = \argmin_{j \neq i} \|\vv{\mu}_{\phi}(\vv{x}^{(j)}) - \vv{\mu}_{\phi}(\vv{x}^{(i)})\|
\end{equation*}
where $\vv{\mu}_{\phi}(\vv{x})$ is the mean of $q_{\phi}(\vv{z}|\vv{x})$, and $\|\cdot\|$ is the L2 norm.
These two posteriors $q_{\phi}(\vv{z}|\vv{x}^{(i)})$ and $q_{\phi}(\vv{z}|\vv{x}^{(j)})$ are called a closest pair of $q_{\phi}(\vv{z}|\vv{x})$.

\paragraph{Distance of a pair of $q_{\phi}(\vv{z}|\vv{x})$}

The distance $d_{ij}$ of a closest pair $q_{\phi}(\vv{z}|\vv{x}^{(i)})$ and $q_{\phi}(\vv{z}|\vv{x}^{(j)})$ is:
\begin{align*}
	\vv{d}_{ij} &= \vv{\mu}_{\phi}(\vv{x}^{(j)}) - \vv{\mu}_{\phi}(\vv{x}^{(i)}) \\
	d_{ij} &= \|\vv{d}_{ij}\|
\end{align*}

\paragraph{Normalized distance of a pair of $q_{\phi}(\vv{z}|\vv{x})$}

For each closest pair $q_{\phi}(\vv{z}|\vv{x}^{(i)})$ and $q_{\phi}(\vv{z}|\vv{x}^{(j)})$, we compute its \textit{normalized distance} $\widetilde{d_{ij}}$ by:
\begin{equation}
	\widetilde{d_{ij}} = \frac{2 d_{ij}}{\mathrm{Std}[i;j] + \mathrm{Std}[j;i]}
\end{equation}
%We assume for each $q_{\phi}(\vv{z}|\vv{x}^{(i)})$, there should exist only one $q_{\phi}(\vv{z}|\vv{x}^{(j)})$ being its nearest neighbor.
%Then the standard deviation for each $\vv{x}^{(i)}$ along the nearest pair distance is defined as:
%\begin{equation*}
%	\mathrm{Std}_i = \sqrt{\Var_{q_{\phi}(\vv{z}|\vv{x}^{(i)})}\left[\left(\vv{z} - \vv{\mu}_{\theta}(\vv{x}^{(i)})\right) \cdot \frac{\vv{d}_{ij}}{d_{ij}} \right]}
%\end{equation*}
$\mathrm{Std}[i;j]$ is formulated as:
\begin{equation}
	\mathrm{Std}[i;j] = \sqrt{\Var_{q_{\phi}(\vv{z}|\vv{x}^{(i)})}\left[\left(\vv{z} - \vv{\mu}_{\phi}(\vv{x}^{(i)})\right) \cdot \frac{\vv{d}_{ij}}{d_{ij}} \right]}
\end{equation}
where $\Var_{q_{\phi}(\vv{z}|\vv{x}^{(i)})}[f(\vv{z})]$ is the variance of $f(\vv{z})$ \wrt{} $q_{\phi}(\vv{z}|\vv{x}^{(i)})$.  We use 1,000 samples to estimate each $\mathrm{Std}[i;j]$.
Roughly speaking, the \textit{normalized distance} $\widetilde{d_{ij}}$ can be viewed as ``distance/std'' along the direction of $\vv{d}_{ij}$, which indicates the scale of the ``hole'' between $q_{\phi}(\vv{z}|\vv{x}^{(i)})$ and $q_{\phi}(\vv{z}|\vv{x}^{(j)})$.

\FloatBarrier

\subsection[Discussions about KL(q(z)||p(z))]{Discussions about $\KLD[q_{\phi}(\vv{z})\|p_{\lambda}(\vv{z})]$}
\label{sec:appendix-kl-qz-pz}
The KL divergence between the \AggregatedPosteriorI{} and the prior, \ie{}, $\KLD[q_{\phi}(\vv{z})\|p_{\lambda}(\vv{z})]$, is first analyzed by \citet{hoffmanElboSurgeryAnother2016} as one component of the ELBO decomposition \eqref{eqn:elbo-2}.
Since $q_{\phi}(\vv{z}) = \int q_{\phi}(\vv{z}|\vv{x})\,p^{\star}(\vv{x})\,\dd{\vv{x}}$ and $p_{\lambda}(\vv{z}) = \int p_{\theta}(\vv{z}|\vv{x})\,p_{\theta}(\vv{x})\,\dd{\vv{x}}$, \citet{roscaDistributionMatchingVariational2018} used this KL divergence as a metric to quantify the approximation quality of both $p_{\theta}(\vv{x})$ to $p^{\star}(\vv{x})$, and $q_{\phi}(\vv{z}|\vv{x})$ to $p_{\theta}(\vv{z}|\vv{x})$.
As we are focusing on learning the prior, $\KLD[q_{\phi}(\vv{z})\|p_{\lambda}(\vv{z})]$ can in turn be a metric for quantifying whether $p_{\lambda}(\vv{z})$ is close enough to the \AggregatedPosteriorI{} $q_{\phi}(\vv{z})$.

The evaluation of $\KLD[q_{\phi}(\vv{z})\|p_{\lambda}(\vv{z})]$, however, is not an easy task.
One way to estimate the KL divergence of two arbitrary distributions is the density ratio trick~\citep{sugiyamaDensityRatioEstimation2012,meschederAdversarialVariationalBayes2017}, where $\KLD[q_{\phi}(\vv{z})\|p_{\lambda}(\vv{z})]$ is estimated by a separately trained neural network classifier.
However, \citet{roscaDistributionMatchingVariational2018} revealed that such approach can under-estimate $\KLD[q_{\phi}(\vv{z})\|p_{\lambda}(\vv{z})]$, and the training may even diverge when $\vv{z}$ has hundreds of or more dimensions.

Another approach is to directly estimate $\KLD[q_{\phi}(\vv{z})\|p_{\lambda}(\vv{z})]$ by Monte Carlo integration.  The KL divergence can be rewritten into the following form:
\begin{align*}
	\KLD[q_{\phi}(\vv{z})\|p_{\lambda}(\vv{z})] 
	&= \int q_{\phi}(\vv{z}) \log \frac{q_{\phi}(\vv{z})}{p_{\lambda}(\vv{z})}\,\dd{\vv{z}} \\
	&= \int\left(\int q_{\phi}(\vv{z}|\vv{x})\,p^{\star}(\vv{x})\,\dd{\vv{x}}\right)	\log \frac{\int q_{\phi}(\vv{z}|\vv{x}')\,p^{\star}(\vv{x}')\,\dd{\vv{x}'}}{p_{\lambda}(\vv{z})}\,\dd{\vv{z}} \\
	&= \int p^{\star}(\vv{x}) \int q_{\phi}(\vv{z}|\vv{x}) \log \frac{\int q_{\phi}(\vv{z}|\vv{x}')\,p^{\star}(\vv{x}')\,\dd{\vv{x}'}}{p_{\lambda}(\vv{z})}\,\dd{\vv{z}}\,\dd{\vv{x}} \\
	&= \E_{p^{\star}(\vv{x})}\,\E_{q_{\phi}(\vv{z}|\vv{x})} \left[ \log \E_{p^{\star}(\vv{x}')}\left[q_{\phi}(\vv{z}|\vv{x}')\right] - \log p_{\lambda}(\vv{z}) \right] \numberthis\label{eqn:kl-qz-pz-expectation-form}
\end{align*}
\citet{roscaDistributionMatchingVariational2018} has already proposed a Monte Carlo based algorithm to estimate $\KLD[q_{\phi}(\vv{z})\|p_{\lambda}(\vv{z})]$, in case that the training data is statically binarized.
Since we mainly use dynamically binarized datasets, we slightly modified the algorithm according to \cref{eqn:kl-qz-pz-expectation-form}, to allow sampling multiple $\vv{x}$ from each image.
Our algorithm is:
\begin{algorithm}[h!]
	\caption{Pseudocode for estimating $\KLD[q_{\phi}(\vv{z})\|p_{\lambda}(\vv{z})]$ (denoted as marginal\_kl) on dynamically binarized dataset, where $\vv{\mu}$ is the pixel intensities of each original image.}
	\label{algo:estimate-kl-qz-pz}

	\begin{algorithmic}
		\State x\_samples = []
		
		\For{$\vv{\mu}$ in training dataset}
			\For{$i = 1 \dots n_x$}
				\State sample $\vv{x}$ from $\mathrm{Bernoulli}\myleft(\vv{\mu}\myright)$
				\State append $\vv{x}$ to x\_samples
			\EndFor
		\EndFor
		\State marginal\_kl = 0
		\For{$\vv{x}$ in x\_samples}
			\For{$i = 1 \dots n_z$}
				\State sample $\vv{z}$ from $q_{\phi}(\vv{z}|\vv{x})$
				\State posterior\_list = []
				\For{$\vv{x}'$ in x\_samples}
					\State append $\log q_{\phi}(\vv{z}|\vv{x}')$ to posterior\_list
				\EndFor
				\State $\log q_{\phi}(\vv{z}) = \mathrm{LogMeanExp}\myleft(\text{posterior\_list}\myright)$
				\State $\text{marginal\_kl} = \text{marginal\_kl} + \log q_{\phi}(\vv{z}) - \log p_{\lambda}(\vv{z})$
			\EndFor
		\EndFor
		\State $\text{marginal\_kl} = \text{marginal\_kl} / (\mathrm{len}(\text{x\_samples}) \times n_z)$
	\end{algorithmic}
\end{algorithm}

Surprisingly, we find that increasing $n_x$ will cause the estimated $\KLD[q_{\phi}(\vv{z})\|p_{\lambda}(\vv{z})]$ to decrease on MNIST, see \cref{tbl:kl-qz-pz-MNIST}.
Given that our $n_z = 10$ for all $n_x$, the number of our sampled $\vv{z}$ (even when $n_x = 1$) is $10 \times 60,000 = 6 \times 10^5$, which should not be too small, since \citet{roscaDistributionMatchingVariational2018} only used $10^6$ $\vv{z}$ in their experiments.
When $n_x = 8$, the number of $\vv{z}$ is $8 \times 10 \times 60,000 = 4.8 \times 10^6$, which is 4.8x larger than \citet{roscaDistributionMatchingVariational2018}, not to mention the inner expectation $\E_{p^{\star}(\vv{x}')}\left[q_{\phi}(\vv{z}|\vv{x}')\right]$ is estimated with 8x larger number of $\vv{x}'$.
There should be in total 38.4x larger number of $\log q_{\phi}(\vv{z}|\vv{x}')$ computed for estimating $\KLD[q_{\phi}(\vv{z})\|p_{\lambda}(\vv{z})]$ than \citet{roscaDistributionMatchingVariational2018}, which has already costed about 3 days on 4 GTX 1080 Ti graphical cards.
We believe such a large number of Monte Carlo samples should be sufficient for any well-defined algorithm.

\begin{table}[!h]
	\centering
	\caption{$\KLD[q_{\phi}(\vv{z})\|p_{\lambda}(\vv{z})]$ of a \ResnetVAE{} with \RealNVP{} prior trained on MNIST, estimated by \cref{algo:estimate-kl-qz-pz}.  $n_z = 10$ for all $n_x$.  We only tried $n_x$ for up to 8, due to the growing computation time of $O\myleft(n_x^2\myright)$.}\label{tbl:kl-qz-pz-MNIST}
	\bgroup\setlength{\tabcolsep}{.75em}\begin{tabular}{cccccc}
  \toprule
    $n_x$ & 1 & 2 & 3 & 5 & 8 \\
    \midrule
  $\KLD[q_{\phi}(\vv{z})\|p_{\lambda}(\vv{z})]$
  &
        15.279
 &
        14.623
 &
        14.254
 &
        13.796
 &
		13.392
    \\
  \bottomrule
\end{tabular}
\egroup
\end{table}

According to the above observation, we suspect there must be some flaw in \cref{algo:estimate-kl-qz-pz}.
Because of this, we do not adopt \cref{algo:estimate-kl-qz-pz} to estimate $\KLD[q_{\phi}(\vv{z})\|p_{\lambda}(\vv{z})]$.

Since no mature method has been published to estimate $\KLD[q_{\phi}(\vv{z})\|p_{\lambda}(\vv{z})]$ yet, we decide not to use $\KLD[q_{\phi}(\vv{z})\|p_{\lambda}(\vv{z})]$ to measure how our learned $p_{\lambda}(\vv{z})$ approximates the \AggregatedPosteriorI{} $q_{\phi}(\vv{z})$.

%The problem turns out to be the inadequately used Monte Carlo estimator for the inner expectation.
%
%We start with the following lemma.  A similar theorem has been proved by \citet{burdaImportanceWeightedAutoencoders2015}.
%\begin{lemma}\label{lemma-B1}
%	For a given distribution $p(x)$, if we estimate $\E_p\left[f(x)\right]$ and $\log \E_p\left[f(x)\right]$ using the following Monte Carlo estimators:
%	\begin{equation*}
%		I_1 = \frac{1}{N} \sum_{i=1}^N f(x^i), \quad
%		I_2 = \log \frac{1}{N} \sum_{i=1}^N f(x^i)
%	\end{equation*}
%	where $x^i$ is the $i$-th sample from $p(x)$. We have:
%	\begin{equation*}
%		\E_p\left[I_1\right] = \E_p\left[f(x)\right], \quad
%		\E_p\left[I_2\right] \leq \log \E_p\left[f(x)\right]
%	\end{equation*}
%	for any finite $N$.
%\end{lemma}
%\begin{proof}
%	For $I_1$, we have:
%	\begin{equation*}
%		\E\left[I_1\right] = \E_p\left[\frac{1}{N} \sum_{i=1}^N f(x^i)\right] 
%			= \frac{1}{N} \sum_{i=1}^N \, \E_p\left[f(x^i)\right] 
%			= \frac{1}{N} \sum_{i=1}^N \, \E_p\left[f(x)\right] 
%			= \E_p\left[f(x)\right]
%	\end{equation*}
%	where we used the fact that $x^i \sim p(x)$, thus $\E_p\left[f(x^i)\right] = \E_p\left[f(x)\right]$.
%	
%	For $I_2$, we have:
%	\begin{align*}
%		\E_p\left[I_2\right] &= \E_p\left[\log \frac{1}{N} \sum_{i=1}^N f(x^i)\right] \leq \log \E_p\left[\frac{1}{N} \sum_{i=1}^N f(x^i)\right] \\
%		&= \log \frac{1}{N} \sum_{i=1}^N \E_p\left[f(x^i)\right] = \log \frac{1}{N} \sum_{i=1}^N \E_p\left[f(x)\right] = \log \E_p\left[f(x)\right]
%	\end{align*}
%
%\end{proof}

\FloatBarrier

\subsection{Additional quantitative results}
\label{sec:appendix-quantitative-results}
\begin{table}[!h]
	\centering
	\caption{Test NLL of different models, with both prior and posterior being Gaussian (``standard''), only posterior being Real NVP (``RNVP $q(z|x)$''), and only prior being Real NVP (``RNVP $p(z)$'').  Flow depth $K=20$.}\label{tbl:main-experiments-nll-with-std}
	\small
	\bgroup\setlength{\tabcolsep}{.5em}
\begin{tabular}{cccccc}
  \toprule
  & & \multicolumn{4}{c}{\tablehr{Datasets}} \\
  \cmidrule(lr){3-6}
  \multicolumn{2}{c}{\tablehr{Models}}  & \StaticMNIST{}  & \MNIST{}  & \FashionMNIST{}  & \Omniglot{}  \\
  \midrule
    \multirow{3}{*}{\DenseVAE{}} &
      standard 
 &
88.84 $\pm$ 0.05 &
84.48 $\pm$ 0.03 &
228.60 $\pm$ 0.03 &
106.42 $\pm$ 0.14      \\
 &       RNVP \(q(z|x)\) 
 &
86.07 $\pm$ 0.11 &
82.53 $\pm$ 0.00 &
227.79 $\pm$ 0.01 &
102.97 $\pm$ 0.06      \\
 &       RNVP \(p(z)\) 
 &
\textbf{84.87 $\pm$ 0.05} &
\textbf{80.43 $\pm$ 0.01} &
\textbf{226.11 $\pm$ 0.02} &
\textbf{102.19 $\pm$ 0.12}      \\
\cmidrule(lr){1-6}    \multirow{3}{*}{\ConvVAE{}} &
      standard 
 &
83.63 $\pm$ 0.01 &
82.14 $\pm$ 0.01 &
227.51 $\pm$ 0.08 &
97.87 $\pm$ 0.02      \\
 &       RNVP \(q(z|x)\) 
 &
81.11 $\pm$ 0.03 &
80.09 $\pm$ 0.01 &
226.03 $\pm$ 0.00 &
94.90 $\pm$ 0.03      \\
 &       RNVP \(p(z)\) 
 &
\textbf{80.06 $\pm$ 0.07} &
\textbf{78.67 $\pm$ 0.01} &
\textbf{224.65 $\pm$ 0.01} &
\textbf{93.68 $\pm$ 0.01}      \\
\cmidrule(lr){1-6}    \multirow{3}{*}{\ResnetVAE{}} &
      standard 
 &
82.95 $\pm$ 0.09 &
81.07 $\pm$ 0.03 &
226.17 $\pm$ 0.05 &
96.99 $\pm$ 0.04      \\
 &       RNVP \(q(z|x)\) 
 &
80.97 $\pm$ 0.05 &
79.53 $\pm$ 0.03 &
225.02 $\pm$ 0.01 &
94.30 $\pm$ 0.02      \\
 &       RNVP \(p(z)\) 
 &
\textbf{79.99 $\pm$ 0.02} &
\textbf{78.58 $\pm$ 0.01} &
\textbf{224.09 $\pm$ 0.01} &
\textbf{93.61 $\pm$ 0.04}      \\
\cmidrule(lr){1-6}    \multirow{3}{*}{\PixelVAE{}} &
      standard 
 &
79.47 $\pm$ 0.02 &
78.64 $\pm$ 0.02 &
224.22 $\pm$ 0.06 &
89.83 $\pm$ 0.04      \\
 &       RNVP \(q(z|x)\) 
 &
79.09 $\pm$ 0.01 &
78.41 $\pm$ 0.01 &
223.81 $\pm$ 0.00 &
89.69 $\pm$ 0.01      \\
 &       RNVP \(p(z)\) 
 &
\textbf{78.92 $\pm$ 0.02} &
\textbf{78.15 $\pm$ 0.04} &
\textbf{223.40 $\pm$ 0.07} &
\textbf{89.61 $\pm$ 0.03}      \\
  \bottomrule
\end{tabular}
        \egroup
\end{table}

\begin{table}[!h]
	\centering
	\caption{Test NLL of \ResnetVAE{}, with only Real NVP posterior (``RNVP $q(z|x)$''), only Real NVP prior (``RNVP $p(z)$''), and both Real NVP prior \& posterior (``both'').  Flow depth $K=20$.}\label{tbl:dual-rnvp-nll-with-std}
    \small
	\bgroup\setlength{\tabcolsep}{.5em}
\begin{tabular}{cccc}
  \toprule
  & \multicolumn{3}{c}{\tablehr{\ResnetVAE}} \\
  \cmidrule(lr){2-4}
  \tablehr{Datasets}  & \parbox[c]{3em}{RNVP \\ \centering \(q(z|x)\)}  & \parbox[c]{3em}{RNVP \\ \centering \(p(z)\)}  & both  \\
  \midrule
    \StaticMNIST{} 
 &
80.97 $\pm$ 0.05 &
79.99 $\pm$ 0.02 &
\textbf{79.87 $\pm$ 0.04}    \\
    \MNIST{} 
 &
79.53 $\pm$ 0.03 &
78.58 $\pm$ 0.01 &
\textbf{78.56 $\pm$ 0.01}    \\
    \FashionMNIST{} 
 &
225.02 $\pm$ 0.01 &
224.09 $\pm$ 0.01 &
\textbf{224.08 $\pm$ 0.02}    \\
    \Omniglot{} 
 &
94.30 $\pm$ 0.02 &
\textbf{93.61 $\pm$ 0.04} &
93.68 $\pm$ 0.04    \\
  \bottomrule
\end{tabular}
        \egroup
\end{table}

\begin{table}[!h]
	\centering
	\caption{Test NLL of \ResnetVAE{}, with Real NVP prior of different flow depth.}
    \small
    \bgroup\setlength{\tabcolsep}{.5em}
\begin{tabular}{ccccc}
  \toprule
  & \multicolumn{4}{c}{\tablehr{Datasets}} \\
  \cmidrule(lr){2-5}
  \tablehr{Flow depth}  & \StaticMNIST{}  & \MNIST{}  & \FashionMNIST{}  & \Omniglot{}  \\
  \midrule
    0 
 &
82.95 $\pm$ 0.09 &
81.07 $\pm$ 0.03 &
226.17 $\pm$ 0.05 &
96.99 $\pm$ 0.04    \\
    1 
 &
81.76 $\pm$ 0.04 &
80.02 $\pm$ 0.02 &
225.27 $\pm$ 0.03 &
96.20 $\pm$ 0.06    \\
    2 
 &
81.30 $\pm$ 0.02 &
79.58 $\pm$ 0.02 &
224.78 $\pm$ 0.02 &
95.35 $\pm$ 0.06    \\
    5 
 &
80.64 $\pm$ 0.06 &
79.09 $\pm$ 0.02 &
224.37 $\pm$ 0.01 &
94.47 $\pm$ 0.01    \\
    10 
 &
80.26 $\pm$ 0.05 &
78.75 $\pm$ 0.01 &
224.18 $\pm$ 0.01 &
93.92 $\pm$ 0.02    \\
    20 
 &
79.99 $\pm$ 0.02 &
78.58 $\pm$ 0.01 &
224.09 $\pm$ 0.01 &
93.61 $\pm$ 0.04    \\
    30 
 &
79.90 $\pm$ 0.05 &
78.52 $\pm$ 0.01 &
\textbf{224.07 $\pm$ 0.01} &
93.53 $\pm$ 0.02    \\
    50 
 &
\textbf{79.84 $\pm$ 0.04} &
\textbf{78.49 $\pm$ 0.01} &
\textbf{224.07 $\pm$ 0.01} &
\textbf{93.52 $\pm$ 0.02}    \\
  \bottomrule
\end{tabular}
        \egroup
\end{table}

\newcommand{\ComparePreviousTable}[2]{
\begin{table}[!h]
	\centering
	\caption{Test NLL on #2.  ``$^\TwoZMark$'' and ``$^\ThreeOrMoreZMark$'' has the same meaning as \cref{tbl:static-mnist-nll-comparison}.}\label{tbl:appendix-#1-nll-comparison}
	\begin{minipage}{.65\linewidth}
	    \small
		\input{metrics/dist/compare_previous/#1_nll_with_std}
	\end{minipage}
\end{table}
}

\begin{table}[!h]
	\centering
	\caption{Test NLL on \StaticMNIST{}.  ``$^\TwoZMark$'' and ``$^\ThreeOrMoreZMark$'' has the same meaning as \cref{tbl:static-mnist-nll-comparison}.}\label{tbl:appendix-static_mnist-nll-comparison}
	\begin{minipage}{.65\linewidth}
	    \small
		
\begin{tabularx}{\linewidth}{Xr}
  \toprule
  Model & NLL \\
    \midrule
\multicolumn{2}{l}{\textit{Models without PixelCNN decoder}} \\    ConvHVAE + Lars prior$^\TwoZMark$~[\citenum{bauerResampledPriorsVariational2018}] & 81.70 \\
    ConvHVAE + VampPrior$^\TwoZMark$~[\citenum{tomczakVAEVampPrior2018}] & 81.09 \\
    VAE + IAF$^\ThreeOrMoreZMark$~[\citenum{kingmaImprovedVariationalInference2016}] & 79.88 \\
    BIVA$^\ThreeOrMoreZMark$~[\citenum{maaloeBIVAVeryDeep2019}] & \textbf{78.59} \\
    \textit{Our \ConvVAE{} + RNVP $p(z)$, $K=50$} & 80.09 $\pm$ 0.01 \\
    \textit{Our \ResnetVAE{} + RNVP $p(z)$, $K=50$} & 79.84 $\pm$ 0.04 \\
    \midrule
\multicolumn{2}{l}{\textit{Models with PixelCNN decoder}} \\    VLAE$^\ThreeOrMoreZMark$[\citenum{chenVariationalLossyAutoencoder2016}] & 79.03 \\
    PixelHVAE + VampPrior$^\TwoZMark$~[\citenum{tomczakVAEVampPrior2018}] & 79.78 \\
    \textit{Our \PixelVAE{} + RNVP $p(z)$, $K=50$} & \textbf{79.01 $\pm$ 0.03} \\
  \bottomrule
\end{tabularx}
        
	\end{minipage}
\end{table}

\begin{table}[!h]
	\centering
	\caption{Test NLL on \MNIST{}.  ``$^\TwoZMark$'' and ``$^\ThreeOrMoreZMark$'' has the same meaning as \cref{tbl:static-mnist-nll-comparison}.}\label{tbl:appendix-mnist-nll-comparison}
	\begin{minipage}{.65\linewidth}
	    \small
		
\begin{tabularx}{\linewidth}{Xr}
  \toprule
  Model & NLL \\
    \midrule
\multicolumn{2}{l}{\textit{Models without PixelCNN decoder}} \\    ConvHVAE + Lars prior$^\TwoZMark$~[\citenum{bauerResampledPriorsVariational2018}] & 80.30 \\
    ConvHVAE + VampPrior$^\TwoZMark$~[\citenum{tomczakVAEVampPrior2018}] & 79.75 \\
    VAE + IAF$^\ThreeOrMoreZMark$~[\citenum{kingmaImprovedVariationalInference2016}] & 79.10 $\pm$ 0.07 \\
    BIVA$^\ThreeOrMoreZMark$~[\citenum{maaloeBIVAVeryDeep2019}] & \textbf{78.41} \\
    \textit{Our \ConvVAE{} + RNVP $p(z)$, $K=50$} & 78.61 $\pm$ 0.01 \\
    \textit{Our \ResnetVAE{} + RNVP $p(z)$, $K=50$} & 78.49 $\pm$ 0.01 \\
    \midrule
\multicolumn{2}{l}{\textit{Models with PixelCNN decoder}} \\    VLAE$^\ThreeOrMoreZMark$~[\citenum{chenVariationalLossyAutoencoder2016}] & 78.53 \\
    PixelVAE$^\TwoZMark$~[\citenum{gulrajaniPixelvaeLatentVariable2016}] & 79.02 \\
    PixelHVAE + VampPrior$^\TwoZMark$~[\citenum{tomczakVAEVampPrior2018}] & 78.45 \\
    \textit{Our \PixelVAE{} + RNVP $p(z)$, $K=50$} & \textbf{78.12 $\pm$ 0.04} \\
  \bottomrule
\end{tabularx}
        
	\end{minipage}
\end{table}

\begin{table}[!h]
	\centering
	\caption{Test NLL on \Omniglot{}.  ``$^\TwoZMark$'' and ``$^\ThreeOrMoreZMark$'' has the same meaning as \cref{tbl:static-mnist-nll-comparison}.}\label{tbl:appendix-omniglot-nll-comparison}
	\begin{minipage}{.65\linewidth}
	    \small
		
\begin{tabularx}{\linewidth}{Xr}
  \toprule
  Model & NLL \\
    \midrule
\multicolumn{2}{l}{\textit{Models without PixelCNN decoder}} \\    ConvHVAE + Lars prior$^\TwoZMark$~[\citenum{bauerResampledPriorsVariational2018}] & 97.08 \\
    ConvHVAE + VampPrior$^\TwoZMark$~[\citenum{tomczakVAEVampPrior2018}] & 97.56 \\
    BIVA$^\ThreeOrMoreZMark$~[\citenum{maaloeBIVAVeryDeep2019}] & \textbf{91.34} \\
    \textit{Our \ConvVAE{} + RNVP $p(z)$, $K=50$} & 93.62 $\pm$ 0.02 \\
    \textit{Our \ResnetVAE{} + RNVP $p(z)$, $K=50$} & 93.52 $\pm$ 0.02 \\
    \midrule
\multicolumn{2}{l}{\textit{Models with PixelCNN decoder}} \\    VLAE$^\ThreeOrMoreZMark$~[\citenum{chenVariationalLossyAutoencoder2016}] & 89.83 \\
    PixelHVAE + VampPrior$^\TwoZMark$~[\citenum{tomczakVAEVampPrior2018}] & 89.76 \\
    \textit{Our \PixelVAE{} + RNVP $p(z)$, $K=50$} & \textbf{89.60 $\pm$ 0.01} \\
  \bottomrule
\end{tabularx}
        
	\end{minipage}
\end{table}

\begin{table}[!h]
	\centering
	\caption{Test NLL on \FashionMNIST{}.  ``$^\TwoZMark$'' and ``$^\ThreeOrMoreZMark$'' has the same meaning as \cref{tbl:static-mnist-nll-comparison}.}\label{tbl:appendix-fashion_mnist-nll-comparison}
	\begin{minipage}{.65\linewidth}
	    \small
		
\begin{tabularx}{\linewidth}{Xr}
  \toprule
  Model & NLL \\
    \midrule
\multicolumn{2}{l}{\textit{Models without PixelCNN decoder}} \\    ConvHVAE + Lars prior$^\TwoZMark$~[\citenum{bauerResampledPriorsVariational2018}] & 225.92 \\
    \textit{Our \ConvVAE{} + RNVP $p(z)$, $K=50$} & 224.64 $\pm$ 0.01 \\
    \textit{Our \ResnetVAE{} + RNVP $p(z)$, $K=50$} & \textbf{224.07 $\pm$ 0.01} \\
    \midrule
\multicolumn{2}{l}{\textit{Models with PixelCNN decoder}} \\    \textit{Our \PixelVAE{} + RNVP $p(z)$, $K=50$} & \textbf{223.36 $\pm$ 0.06} \\
  \bottomrule
\end{tabularx}
        
	\end{minipage}
\end{table}

\begin{table}[!h]
	\centering
	\caption{Test NLL of \ResnetVAE{}, with prior trained by: \textit{joint} training, \textit{iterative} training, \textit{post-hoc} training, and standard VAE (``none'') as reference.  Flow depth $K=20$.}
    \small
    \bgroup\setlength{\tabcolsep}{.5em}
\begin{tabular}{ccccc}
  \toprule
  & \multicolumn{4}{c}{\tablehr{\ResnetVAE}} \\
  \cmidrule(lr){2-5}
  \tablehr{Datasets}  & joint  & iterative  & post-hoc  & none  \\
  \midrule
    \StaticMNIST{} 
 &
\textbf{79.99 $\pm$ 0.02} &
80.63 $\pm$ 0.02 &
80.86 $\pm$ 0.04 &
82.95 $\pm$ 0.09    \\
    \MNIST{} 
 &
\textbf{78.58 $\pm$ 0.01} &
79.61 $\pm$ 0.01 &
79.90 $\pm$ 0.04 &
81.07 $\pm$ 0.03    \\
    \FashionMNIST{} 
 &
\textbf{224.09 $\pm$ 0.01} &
224.88 $\pm$ 0.02 &
225.22 $\pm$ 0.01 &
226.17 $\pm$ 0.05    \\
    \Omniglot{} 
 &
\textbf{93.61 $\pm$ 0.04} &
94.43 $\pm$ 0.11 &
94.87 $\pm$ 0.05 &
96.99 $\pm$ 0.04    \\
  \bottomrule
\end{tabular}
        \egroup
\end{table}

\FloatBarrier

\subsection{Additional qualitative results}
\label{sec:appendix-qualitative-results}

\newcommand{\SampleDetails}[3]{
\begin{figure}[!h]
	\centering
	\subfigure[standard]{
	\begin{minipage}{0.32\linewidth}
		\flushleft
		\includegraphics[width=\linewidth]{figures/samples/#2_#1_Normal_Normal_10x10.png}
		\vspace{.1em}
	\end{minipage}}\hfill
	\subfigure[\RealNVP{} $q(z|x)$]{\begin{minipage}{0.32\linewidth}
		\centering
		\includegraphics[width=\linewidth]{figures/samples/#2_#1_Normal_RNVP_10x10.png}
		\vspace{.1em}
	\end{minipage}}\hfill
	\subfigure[\RealNVP{} $p(z)$]{\begin{minipage}{0.32\linewidth}
		\flushright
		\includegraphics[width=\linewidth]{figures/samples/#2_#1_RNVP_Normal_10x10.png}
		\vspace{.1em}
	\end{minipage}}
	\caption{Samples from #3 trained on #1.  \TheLastColumnSentence{10x10}}
\end{figure}
}

\SampleDetails{MNIST}{ResnetVAE}{\ResnetVAE{}}
\SampleDetails{FashionMNIST}{ResnetVAE}{\ResnetVAE{}}
\SampleDetails{Omniglot}{ResnetVAE}{\ResnetVAE{}}

%\SampleDetails{MNIST}{PixelVAE}{\PixelVAE{}}
%\SampleDetails{FashionMNIST}{PixelVAE}{\PixelVAE{}}
%\SampleDetails{Omniglot}{PixelVAE}{\PixelVAE{}}

\FloatBarrier

\subsection{Additional results: improved reconstruction loss and other experimental results with learned prior}
\label{sec:appendix-influence-of-learning-the-prior}
\newcommand{\KLReconsTable}[2]{
  \begin{table}[!h]
	\centering
	\caption{Average test \ELBOTerm{}, \ReconsTerm{} (``\textit{recons}''), \KLTerm{} (``\textit{kl}'') and \KLzxTerm{} (``$kl_{z|x}$'') of various #2.  Flow depth $K=20$.}\label{tbl:kl-recons-2-#1}
	\small
	\bgroup\setlength{\tabcolsep}{.65em}\input{metrics/dist/kl_recons_2/kl_recons_#1}\egroup
  \end{table}
}

\KLReconsTable{DenseVAE}{\DenseVAE{}}
\KLReconsTable{ConvVAE}{\ConvVAE{}}
\KLReconsTable{ResnetVAE}{\ResnetVAE{}}
\KLReconsTable{PixelVAE}{\PixelVAE{}}

\FloatBarrier

\subsection{Additional results: learned prior on low posterior samples}
\label{sec:appendix-learned-prior-on-low-posterior-samples}
\citet{roscaDistributionMatchingVariational2018} proposed an algorithm to obtain \textit{low posterior samples} ($\vv{z}$ samples which have low likelihoods on $q_{\phi}(\vv{z})$) from a trained VAE, and plotted the histograms of $\log p_{\lambda}(\vv{z})$ evaluated on these $\vv{z}$ samples.
Their algorithm first samples a large number of $\vv{z}$ from the prior $p_{\lambda}(\vv{z})$, then uses Monte Carlo estimator to evaluate $q_{\phi}(\vv{z})$ on these $\vv{z}$ samples, and finally chooses a certain number of $\vv{z}$ with the lowest $q_{\phi}(\vv{z})$ likelihoods as the \textit{low posterior samples}.
They also sampled one $\vv{x}$ from $p_{\theta}(\vv{x}|\vv{z})$ for each low posterior sample $\vv{z}$, plotted the sample means of these $\vv{x}$ and the histograms of ELBO on these $\vv{x}$.
Although we have found their Monte Carlo estimator for $\KLD[q_{\phi}(\vv{z})\|p_{\lambda}(\vv{z})]$ vulnerable (see \cref{sec:appendix-kl-qz-pz}), their \textit{low posterior samples} algorithm only ranks $q_{\phi}(\vv{z})$ for each $\vv{z}$, and the visual results of their algorithm seems plausible.
Thus we think this algorithm should be still convincing enough, and we also use it to obtain \textit{low posterior samples}.

\newcommand{\theStandardVAE}{\textit{standard} \ResnetVAE{}}
\newcommand{\thePosthocVAE}{\textit{post-hoc trained} \ResnetVAE{}}
\newcommand{\theJointVAE}{\textit{jointly trained} \ResnetVAE{}}

To compare the learned prior with unit Gaussian prior on such \textit{low posterior samples}, we first train a \ResnetVAE{} with unit Gaussian prior (denoted as \theStandardVAE{}), and then add a \textit{post-hoc trained} \RealNVP{} prior upon this original \ResnetVAE{} (denoted as \thePosthocVAE{}).
We then obtain 10,000 $\vv{z}$ samples from the \theStandardVAE{}, and choose 100 $\vv{z}$ with the lowest $q_{\phi}(\vv{z})$ among these 10,000 samples, evaluated on \theStandardVAE{}.
Fixing these 100 $\vv{z}$ samples, we plot the histograms of $\log p_{\lambda}(\vv{z})$ \wrt{} \theStandardVAE{} and \thePosthocVAE{}.
We also obtain one $\vv{x}$ sample from $p_{\theta}(\vv{x}|\vv{z})$ for each $\vv{z}$, and plot the histograms of ELBO for each $\vv{x}$ (\wrt{} the two models) and their sample means.
See \cref{fig:low-posterior-samples}.

\begin{figure}[!h]
	\centering
	\includegraphics[width=.8\linewidth]{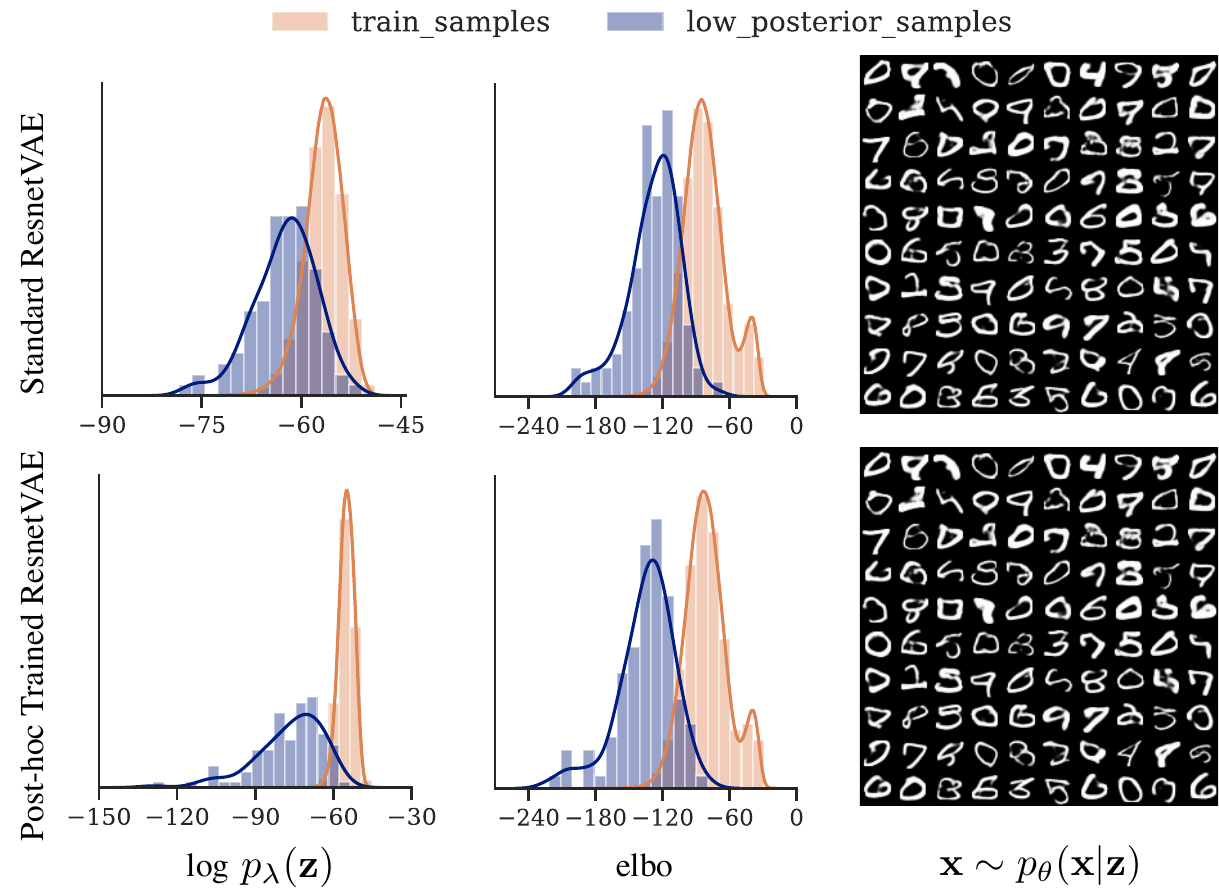}
	\caption{(left) Histograms of $\log p_{\lambda}(\vv{z})$ of the low posterior samples, (middle) histograms of ELBO of $\vv{x}$ samples corresponding to each low posterior sample, and (right) the means of these $\vv{x}$, on \theStandardVAE{} and \thePosthocVAE{}.}\label{fig:low-posterior-samples}
\end{figure}

The $\vv{x}$ samples of \thePosthocVAE{} (bottom right) are the same as those of \theStandardVAE{} (top right), since \thePosthocVAE{} has exactly the same $p_{\theta}(\vv{x}|\vv{z})$ and $q_{\phi}(\vv{z}|\vv{x})$ as \theStandardVAE{}.
However, the \textit{post-hoc trained} prior successfully assigns much lower $\log p_{\lambda}(\vv{z})$ (bottom left) than unit Gaussian prior (top left) on the \textit{low posterior samples}, which suggests that a \textit{post-hoc trained} prior can avoid granting high likelihoods to these samples in the latent space.
Note \thePosthocVAE{} also assigns slightly lower ELBO (bottom middle) than \theStandardVAE{} (top middle) to $\vv{x}$ samples corresponding to these \textit{low posterior samples}.

To verify whether a learned prior can avoid obtaining \textit{low posterior samples} in the first place, we obtained \textit{low posterior samples} from a \ResnetVAE{} with \textit{jointly trained} prior (denoted as \theJointVAE{}), see \cref{fig:low-posterior-samples-joint}.
Compared to \cref{fig:low-posterior-samples}, we can see that $\log p_{\lambda}(\vv{z})$ of these \textit{low posterior samples} and ELBO of the corresponding $\vv{x}$ samples are indeed substantially higher than those of \theStandardVAE{} and \thePosthocVAE{}.
%The variety of these samples are also larger.
However, the visual quality is not perfect, indicating there is still room for improvement.

\begin{figure}[!h]
	\centering
	\includegraphics[width=.8\linewidth]{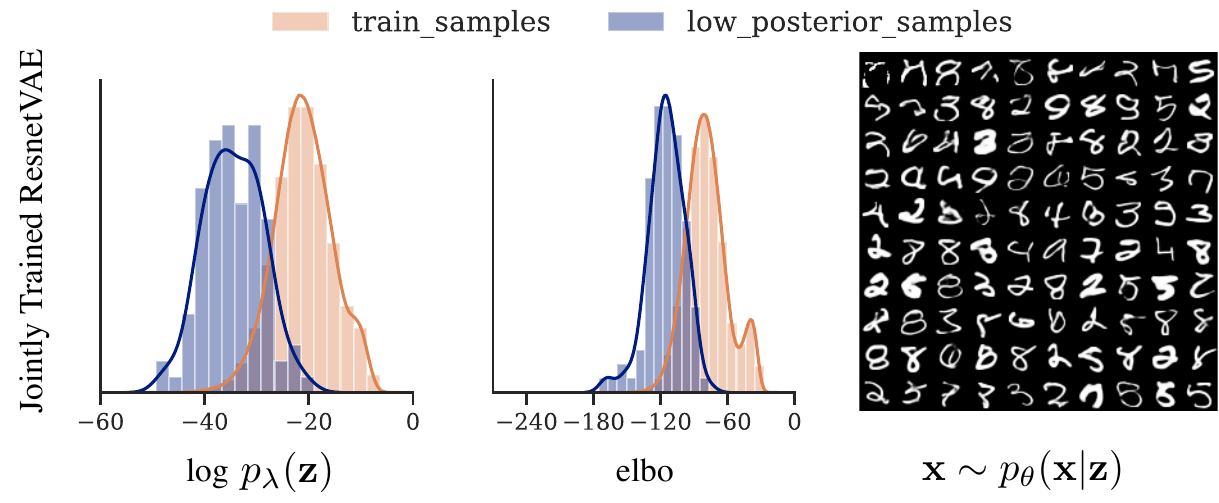}
	\caption{(left) Histograms of $\log p_{\lambda}(\vv{z})$ of the low posterior samples, (middle) histograms of ELBO of $\vv{x}$ samples corresponding to each low posterior sample, and (right) the means of these $\vv{x}$, on \theJointVAE{}.}\label{fig:low-posterior-samples-joint}
\end{figure}

\FloatBarrier

\subsection{Wall-clock times}
\label{sec:appendix-wallclock-times}
\begin{figure}[!h]
	\centering
	\includegraphics[width=.75\linewidth]{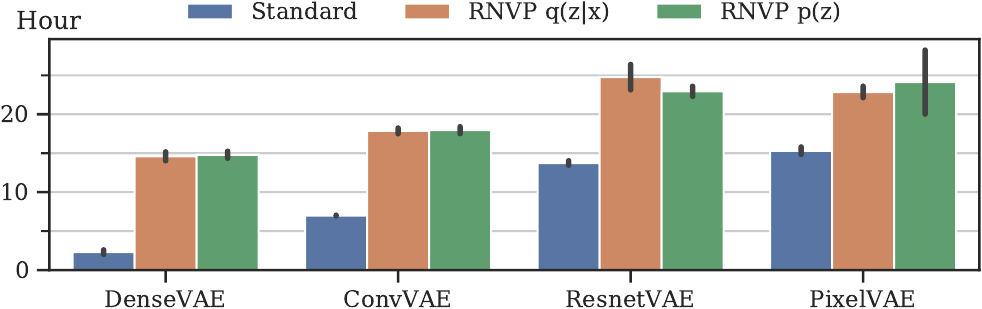}
	\caption{Average training time of various models on MNIST, flow depth $K=20$.  Black sticks are standard deviation bars.  For \PixelVAE{}, training mostly terminates in half way due to early-stopping.}\label{fig:average-time}
\end{figure}

\begin{figure}[!h]
	\centering
	\includegraphics[width=.75\linewidth]{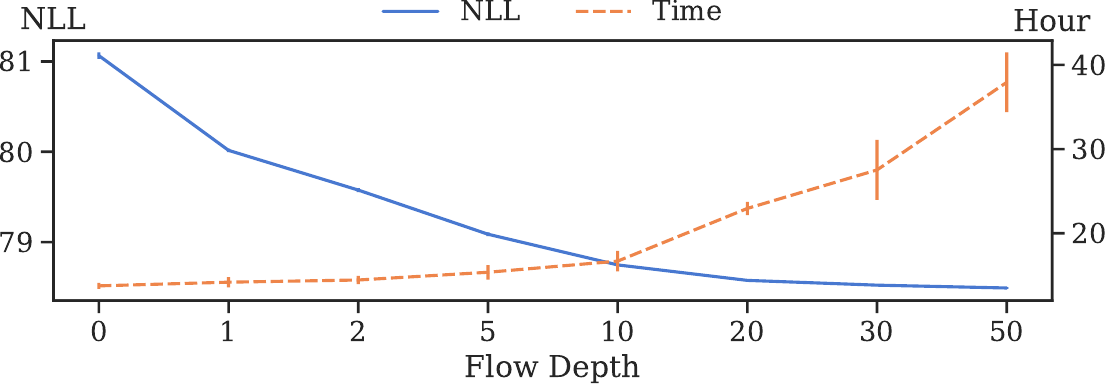}
	\caption{Average training time and test negative log-likelihood (NLL) of \ResnetVAE{} with \RealNVP{} prior of different flow depth.  Vertical sticks are standard deviation bars.}\label{fig:nll-time-with-flow-layers}
\end{figure}

We report the average training time of various models trained on MNIST, see \cref{fig:average-time,fig:nll-time-with-flow-layers}.
Each experiment runs on one GTX 1080 Ti graphical card.
In \cref{fig:average-time}, we can see that the computational cost of \RealNVP{} prior is independent with the model architecture.
For complicated architectures like \ResnetVAE{} and \PixelVAE{}, the cost of adding a \RealNVP{} prior is fairly acceptable, since it can bring large improvement.
In \cref{fig:nll-time-with-flow-layers}, we can see that for $K > 50$, there is likely to be little gain in test NLL, but the computation time will grow even larger.
That's why we do not try larger $K$ in our experiments.

\end{document}